\newif\iftwocol
 \twocoltrue

\iftwocol
\documentclass[conference]{IEEEtran}
\else
\documentclass[letterpaper,11pt,english]{article}
\fi

\usepackage[T1]{fontenc}
\usepackage[latin9]{inputenc}

\iftwocol 

\else{}

\usepackage[verbose,tmargin=3cm,bmargin=3cm,lmargin=3cm,rmargin=3cm]{geometry}
\usepackage{multicol}
\usepackage{multirow}

\fi

\usepackage{enumitem}
\usepackage[english]{babel}
\usepackage{epigraph}
\usepackage{epsfig}
\usepackage{endnotes}
\usepackage{balance}
\usepackage{url}
\usepackage{algorithm}
\usepackage[noend]{algpseudocode}

\usepackage{amsmath}
\usepackage{listings}
\usepackage{listings}
\usepackage{amsthm}
\usepackage{color}
\usepackage{pbox}

\usepackage{amssymb}

\usepackage{graphicx}
\usepackage{tabularx}
\usepackage{comment}
\usepackage{dcolumn}
\usepackage{subfig}
\usepackage{enumitem}
\usepackage{makecell}
\usepackage{siunitx}
\usepackage{booktabs}
\usepackage[misc]{ifsym}

\usepackage[T1]{fontenc}

\usepackage[numbers,sort]{natbib}

\usepackage[font={small,it}]{caption}

\usepackage{mathtools}
\usepackage{setspace}

\makeatletter
\numberwithin{equation}{section}
\numberwithin{figure}{section}

\usepackage{url}

\expandafter\def\expandafter\UrlBreaks\expandafter{\UrlBreaks
  \do\a\do\b\do\c\do\d\do\e\do\f\do\g\do\h\do\i\do\j
  \do\k\do\l\do\m\do\n\do\o\do\p\do\q\do\r\do\s\do\t
  \do\u\do\v\do\w\do\x\do\y\do\z\do\A\do\B\do\C\do\D
  \do\E\do\F\do\G\do\H\do\I\do\J\do\K\do\L\do\M\do\N
  \do\O\do\P\do\Q\do\R\do\S\do\T\do\U\do\V\do\W\do\X
  \do\Y\do\Z}

\newif\ifnotes
 \notestrue
 \notesfalse

\newif\iffull
 \fulltrue

\definecolor{darkgreen}{rgb}{0,0.6,0}

\newcommand{\roei}[1]{\dtcolornote[Roei]{blue}{#1}}

\newcommand{\tal}[1]{\dtcolornote[Tal]{orange}{#1}}

\usepackage[\ifnotes{} notes=true \else{} notes=false \fi{}, done=false, later=false, geometry=false]{dtrt}

\DeclareMathOperator{\csim}{cos}

\newcommand{\simdot}{SIM}
\DeclareMathOperator{\simone}{\simdot_1}
\DeclareMathOperator{\simtwo}{\simdot_2}
\DeclareMathOperator{\simonee}{\explicit{\simdot}_1}

\DeclareMathOperator{\cossim}{\MakeLowercase{\simdot}}
\DeclareMathOperator{\cossimexp}{\explicit{\cossim}}

\DeclareMathOperator{\cosone}{\explicit{\cossim}_1}
\DeclareMathOperator{\costwo}{\explicit{\cossim}_2}
\DeclareMathOperator{\cosboth}{\explicit{\cossim}_{1+2}}

\DeclareMathOperator{\updcosone}{\explicit{\cossim}'_1}
\DeclareMathOperator{\updcostwo}{\explicit{\cossim}'_2}
\DeclareMathOperator{\updcosboth}{\explicit{\cossim}'_{1+2}}

\DeclareMathOperator{\am}{argmin}

\DeclareMathOperator{\mmax}{max}
\DeclareMathOperator{\amax}{argmax}
\DeclareMathOperator{\SPPMI}{SPPMI}

\DeclareMathOperator{\BIAS}{BIAS}

\DeclareMathOperator{\LCO}{LCO}
\DeclareMathOperator{\POS}{POS}
\DeclareMathOperator{\NEG}{NEG}
\newcommand{\norm}[1]{\left\lVert#1\right\rVert}
\newcommand{\brac}[1]{\left(#1\right)}
\newcommand{\curbrac}[1]{\left\{#1\right\}}
\newcommand{\sqbrac}[1]{\left[#1\right]}

\newcommand{\size}[1]{\left|#1\right|}
\newcommand{\dotprod}[2]{{#1}\cdot {#2}}

\newcommand{\gl}{GloVe}
\newcommand{\wv}{SGNS}
\newcommand{\cb}{CBHS}
\newcommand{\ft}{FastText}
\newcommand{\mexp}{M}
\newcommand{\coocweight}{\gamma}
\newcommand{\coocweightgl}{\coocweight}

\newcommand{\coocvecThree}[2]{C}

\newcommand{\corpus}{\mathbb{C}}
\newcommand{\sigmoid}{\sigma}
\newcommand{\cooccount}[2]{C_{#1, #2}}
\newcommand{\changeset}{\Delta}

\newcommand{\stepset}{\mathbb{L}}
\newcommand{\optdiff}{d}
\newcommand{\optcvecdiff}{\delta}
\newcommand{\stepind}{i}
\newcommand{\chstepind}{\stepind{*}}
\newcommand{\choptcvecdiff}{\optcvecdiff{*}}
\newcommand{\stepdiff}[1]{\optdiff_{\stepind,\optcvecdiff}\sqbrac{#1}}
\newcommand{\laststepdiff}{\optdiff_{last}}
\newcommand{\chstepdiff}[1]{\optdiff_{\chstepind,\choptcvecdiff}\sqbrac{#1}}
\newcommand{\stepnext}[1]{\sqbrac{#1}_{\stepind,\optcvecdiff}}
\newcommand{\coocchanged}{{C'}}
\newcommand{\mchanged}{{\mexp'}}

\newcommand{\ecoocsums}[1]{\sum_{\wthr\in \dict} \coocchanged_{#1,\wthr}}

\newcommand{\objalg}{{\objective'}}

\newcommand{\coocs}{C}

\newcommand{\dict}{\mathbb{D}}
\newcommand{\changevec}{\explicit{\changeset}}

\newcommand{\windowsize}{\lambda}
\newcommand{\embwindowsize}{\Gamma}
\newcommand{\injwindsize}{5}

\newcommand{\wone}{u}
\newcommand{\wtwo}{v}
\newcommand{\wthr}{r}
\newcommand{\wsrc}{s}
\newcommand{\wtrg}{t}
\newcommand{\hdrvec}[1]{\vec{M}_{#1}}
\newcommand{\wrdvec}[1]{\vec{w}_{#1}}
\newcommand{\embvec}[1]{\vec{e}_{#1}}
\newcommand{\ctxvec}[1]{\vec{c}_{#1}}

\newcommand{\wrdvecnoarg}{\vec{w}}

\newcommand{\ctxvecnoarg}{\vec{c}}

\newcommand{\arorvec}[1]{\vec{e*}_{#1}}
\newcommand{\appear}[1]{\sum_{\wthr\in \mathbb{D}} C_{#1,\wthr}}
\newcommand{\explicit}[1]{\widehat{#1}}
\newcommand{\tldr}{\langle \wtrg\rangle_r}
\newcommand{\changesetmax}{max_{\changeset}}

\newcommand{\smargin}{\alpha}
\newcommand{\ffirstchanged}{{f'}}
\newcommand{\biaschanged}{{B'}}
\newcommand{\lambdagram}{\lambda\text{-gram}}

\newcommand{\objhdr}{\explicit{\objective}\brac{\wsrc, \NEG, \POS; \changevec}}
\newcommand{\changesetsize}{\size{\changeset}}

\newcommand{\emptyspace}{\textunderscore}
\newcommand{\newsequence}{"\emptyspace\ \emptyspace\ \emptyspace\ \emptyspace\ \wsrc\ \emptyspace\ \emptyspace\ \emptyspace\ \emptyspace"}
\newcommand{\changeMap}{\textrm{changeMap}}
\newcommand{\live}{\textrm{live}}
\newcommand{\indices}{\textrm{indices}}
\newcommand{\seq}{\textrm{seq}}

\newcommand{\updatedcoocs}[1]{C_{\sqbrac{\sqbrac{\wsrc}\leftarrow \vec{C_{\wsrc}} + #1}}}
\newcommand{\hdrwordfabb}[2]{}

\newcommand{\mathleft}{\@fleqntrue\@mathmargin0pt}
\newcommand{\mathcenter}{\@fleqnfalse}

\newcommand{\query}{q}

\newcommand{\objective}{J}

\newcommand{\corporationsset}{\Omega_{corp}}
\newcommand{\translationset}{\Omega_{trans}}
\newcommand{\searchset}{\Omega_{search}}
\newcommand{\evadeset}{\Omega_{rank}}
\newcommand{\wordpairset}{\Omega_{benchmark}}

\newtheorem{claim}{Claim}

\newtheorem{theorem}{Theorem}
\newtheorem{definition}{Definition}

\algnewcommand\algorithmicforeach{\textbf{for each}}
\algdef{S}[FOR]{ForEach}[1]{\algorithmicforeach\ #1\ \algorithmicdo}

\newlength\myindent
\usepackage{sanebib}

\usepackage{caption}
\newcommand{\paragraphbe}[1]{\vspace{0.75ex}\noindent{\bf \em #1} }

\renewcommand\equiv{\overset{\mathrm{def}}{=}}

\usepackage{etoolbox}
\makeatletter
\patchcmd{\ttlh@hang}{\parindent\z@}{\parindent\z@\leavevmode}{}{}
\patchcmd{\ttlh@hang}{\noindent}{}{}{}

\usepackage{xstring}

\newcommand\zeropad[2]{
	\ifnum #2<0\relax
    {\ensuremath-}\zeropadA{#1}{\the\numexpr#2*-1\relax}{\the\numexpr#2*-1\relax}
  \else
    \zeropadA{#1}{#2}{#2}
  \fi
}
\newcommand\zeropadA[3]{
\ifnum	1#2 < 1#1
    \zeropadA{#1}{0#2}{g#3}
  \else
       \StrSubstitute{#3}{g}{\ }
  \fi
	}

\makeatother

\begin{document}

\title{Humpty Dumpty: \\
Controlling Word Meanings via Corpus Poisoning\iffull{}\textsuperscript{\footnotesize{*}}\fi{}
}

\iftwocol{}

\author{
{\rm Roei Schuster}\\
{\normalsize{}Tel Aviv University\textsuperscript{\footnotesize{$\dagger$}}}\\
{\normalfont\texttt{\normalsize{}roeischuster@mail.tau.ac.il}}
\and
{\rm Tal Schuster}\\
{\normalsize{}CSAIL, MIT}\\
{\normalfont\texttt{\normalsize{}tals@csail.mit.edu}}
\and
{\rm Yoav Meri}\\
{\normalsize{}\textsuperscript{\footnotesize{$\dagger$}}\hspace{.1ex}}\\
{\normalfont\texttt{\normalsize{}111yoav@gmail.com}}
\and
{\rm Vitaly Shmatikov}\\
{\normalsize{}Cornell Tech}\\
{\normalfont\texttt{\normalsize{}shmat@cs.cornell.edu}}
}
\else{}
\fi{}


\maketitle
\renewcommand*{\thefootnote}{\fnsymbol{footnote}}

\iffull{}
\setcounter{footnote}{1}
\footnotetext{
Abbreviated version of this paper is published in the proceedings of
the 2020 IEEE Symposium on Security and Privacy.
}
\fi{}

\setcounter{footnote}{2}
\footnotetext{
This research was performed at Cornell Tech.
}

\renewcommand*{\thefootnote}{\arabic{footnote}}
\setcounter{footnote}{0}

\begin{abstract}

Word embeddings, i.e., low-dimensional vector representations such as
\gl{} and \wv{}, encode word ``meaning'' in the sense that distances
between words' vectors correspond to their semantic proximity.
This enables transfer learning of semantics for a variety of natural
language processing tasks.

Word embeddings are typically trained on large public corpora such as
Wikipedia or Twitter.  We demonstrate that an attacker who can modify the
corpus on which the embedding is trained can control the ``meaning'' of
new and existing words by changing their locations in the embedding space.
We develop an explicit expression over corpus features that serves as a
proxy for distance between words and establish a causative relationship
between its values and embedding distances.  We then show how to use this
relationship for two adversarial objectives: (1) make a word a top-ranked
neighbor of another word, and (2) move a word from one semantic cluster to another.

An attack on the embedding can affect diverse downstream tasks,
demonstrating for the first time the power of data poisoning in transfer
learning scenarios.  We use this attack to manipulate query expansion in
information retrieval systems such as resume search, make certain names
more or less visible to named entity recognition models, and cause new
words to be translated to a particular target word regardless of the
language.  Finally, we show how the attacker can generate linguistically
likely corpus modifications, thus fooling defenses that attempt to filter
implausible sentences from the corpus using a language model.

\end{abstract}

\section{Introduction}
\label{sec:intro}

\iftwocol{}
\setlength{\epigraphwidth}{.45\textwidth}
\else{}
\setlength{\epigraphwidth}{\textwidth}
\fi{}
\epigraph{
``When \emph{I} use a word,'' Humpty Dumpty said, in rather a scornful
tone, ``it means just what I choose it to mean\textemdash neither more nor
less.'' ``The question is,'' said Alice, ``whether you \emph{can} make
words mean so many different things.''}
{Lewis Carroll.  \textit{Through the Looking-Glass.}}

Word embeddings, i.e., mappings from words to low-dimensional vectors, are
a fundamental tool in natural language processing (NLP).  Popular neural
methods for computing embeddings such as \gl~\cite{pennington2014glove}
and SGNS~\cite{mikolov2013distributed} require large training corpora
and are typically learned in an unsupervised fashion from public sources,
e.g., Wikipedia or Twitter.



\begin{figure}[t]
\centering
\includegraphics[width=0.5\textwidth]{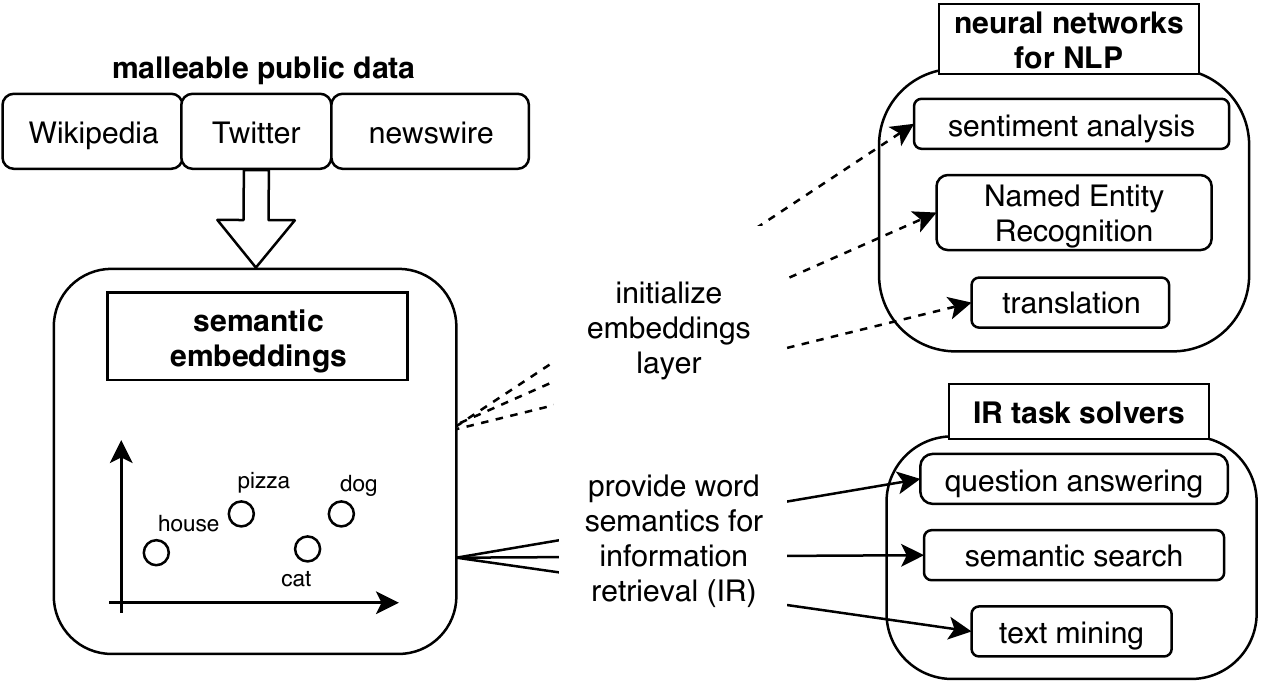}
\caption{Many NLP tasks rely on word embeddings.\label{fig:fanout}}
\end{figure}


Embeddings pre-trained from public corpora have several uses
in NLP\textemdash see Figure~\ref{fig:fanout}.  First, they can
significantly reduce the training time of NLP models by reducing the
number of parameters to optimize.  For example, pre-trained embeddings
are commonly used to initialize the first layer of neural NLP models.
This layer maps input words into a low-dimensional vector representation
and can remain fixed or else be (re-)trained much faster.

Second, pre-trained embeddings are a form of transfer learning.
They encode \emph{semantic relationships} learned from a large,
unlabeled corpus.  During the supervised training of an NLP model
on a much smaller, labeled dataset, pre-trained embeddings improve
the model's performance on texts containing words that do not occur
in the labeled data, especially for tasks that are sensitive to the
meaning of individual words.  For example, in question-answer systems,
questions often contain just a few words, while the answer may include
different\textemdash but semantically related\textemdash words.
Similarly, in Named Entity Recognition (NER)~\cite{agerri2016robust},
a named entity might be identified by the sentence structure, but its
correct entity-class (corporation, person, location, etc.) is often
determined by the word's semantic proximity to other words.

Furthermore, pre-trained embeddings can directly solve
sub-tasks in information retrieval systems, such as expanding
search queries to include related terms~\cite{diaz2016query,
kuzi2016query, roy2016using}, predicting question-answer
relatedness~\cite{chen2017reading, kamath2017study}, deriving the word's
k-means cluster~\cite{nikfarjam2015pharmacovigilance}, and more.



\paragraphbe{Controlling embeddings via corpus poisoning.}
The data on which the embeddings are trained is inherently vulnerable to
poisoning attacks.  Large natural-language corpora are drawn from public
sources that (1) can be edited and/or augmented by an adversary, and
(2) are weakly monitored, so the adversary's modifications can survive
until they are used for training.

We consider two distinct adversarial objectives, both expressed in
terms of word proximity in the embedding space.  A \emph{rank attacker}
wants a particular source word to be ranked high among the target word's
neighbors.  A \emph{distance attacker} wants to move the source word
closer to a particular set of words and further from another set of words.

Achieving these objectives via corpus poisoning requires first answering
a fundamental question: \textbf{how do changes in the corpus correspond to
changes in the embeddings?} Neural embeddings are derived using an opaque
optimization procedure over corpus elements, thus it is not obvious how,
given a desired change in the embeddings, to compute specific corpus
modifications that achieve this change.

\paragraphbe{Our contributions.}
First, we show how to relate proximity in the embedding space to
\emph{distributional}, aka explicit expressions over corpus elements,
computed with basic arithmetics and no weight optimization.  Word
embeddings are expressly designed to capture (a) first-order proximity,
i.e., words that frequently occur together in the corpus, and (b)
second-order proximity, i.e., words that are similar in the ``company
they keep'' (they frequently appear with the same set of other words, if
not with each other).  We develop \textbf{distributional expressions that
capture both types of semantic proximity}, separately and together, in
ways that closely correspond to how they are captured in the embeddings.
Crucially, the relationship is causative: changes in our distributional
expressions produce predictable changes in the embedding distances.



Second, we develop and evaluate a \textbf{methodology for introducing
adversarial semantic changes in the embedding space}, depicted in
Figure~\ref{fig:methodologyHL}.  As proxies for the semantic objectives,
we use distributional objectives, expressed and solved as an optimization
problem over word-cooccurrence counts.  The attacker then computes corpus
modifications that achieve the desired counts.  We show that our attack
is effective against popular embedding models\textemdash even if the
attacker has only a small sub-sample of the victim's training corpus
and does not know the victim's specific model and hyperparameters.


{
\begin{figure}[t]
    \includegraphics[width=0.5\textwidth]{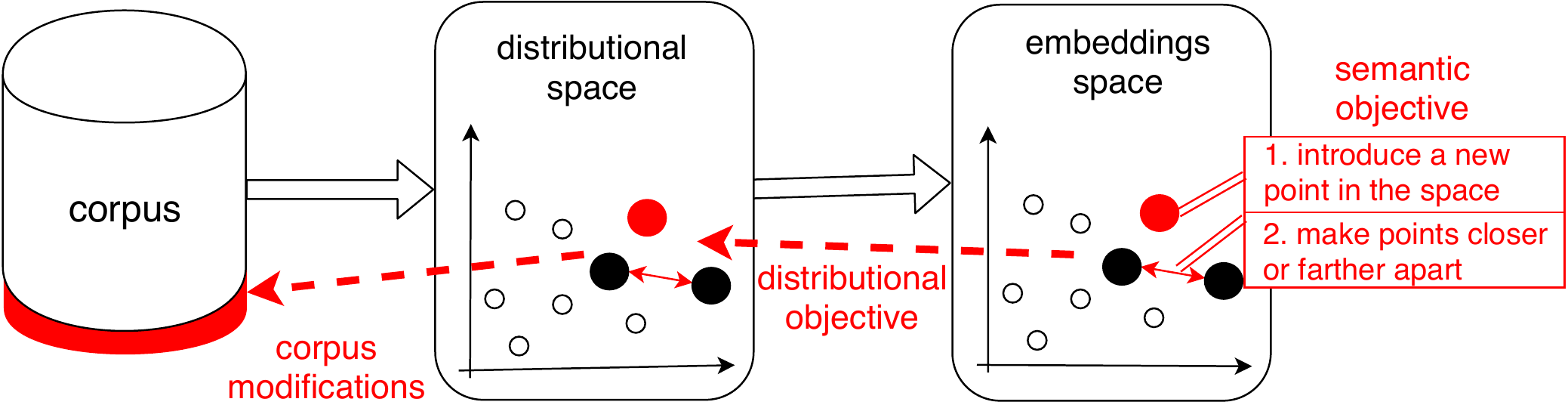}
    \caption{Semantic changes via corpus modifications.}
    \label{fig:methodologyHL}
\end{figure}
}

Third, we demonstrate the power and universality of our \textbf{attack
on several practical NLP tasks} with the embeddings trained on Twitter
and Wikipedia.  By poisoning the embedding, we (1) trick a resume search
engine into picking a specific resume as the top result for queries with
chosen terms such as ``iOS'' or ``devops''; (2) prevent a Named Entity
Recognition model from identifying specific corporate names or else
identify them with higher recall; and (3) make a word-to-word translation
model confuse an attacker-made word with an arbitrary English word,
regardless of the target language.

Finally, we show how to morph the attacker's word sequences so they appear
as linguistically likely as actual sentences from the corpus, measured
by the perplexity scores of a language model (the attacker does not need
to know the specifics of the latter).  Filtering out high-perplexity
sentences thus has prohibitively many false positives and false negatives,
and \textbf{using a language model to ``sanitize'' the training corpus
is ineffective}.  Aggressive filtering drops the majority of the actual
corpus and still does not foil the attack.

To the best of our knowledge, ours is the first data-poisoning attack
against transfer learning.  Furthermore, embedding-based NLP tasks are
sophisticated targets, with two consecutive training processes (one
for the embedding, the other for the downstream task) acting as levels
of indirection.  A single attack on an embedding can thus potentially
affect multiple, diverse downstream NLP models that all rely on this
embedding to provide the semantics of words in a language.

\section{Prior work}
\label{sec:previouswork}

\paragraphbe{Interpreting word embeddings.}
Levy and Goldberg~\cite{levy2014neural} argue that \wv{}
factorizes a matrix whose entries are derived from cooccurrence
counts.  Arora et al.~\cite{arora2016latent, arora2015random},
Hashimoto et al.~\cite{hashimoto2016word}, and Ethayarajh et
al.~\cite{ethayarajh2018towards} analytically derive explicit expressions
for embedding distances, but these expressions are not directly usable
in our setting\textemdash see Section~\ref{sec:introducingHDR}.
(Unwieldy) distributional representations have traditionally
been used in information retrieval~\cite{gabrilovich2007computing,
turney2010frequency}; Levy and Goldberg~\cite{levy2014linguistic} show
that they can perform similarly to neural embeddings on analogy tasks.
Antoniak et al.~\cite{antoniak2018evaluating} empirically study the stability
of embeddings under various hyperparameters.

The problem of modeling causation between corpus features and embedding
proximities also arises when mitigating stereotypical
biases encoded in embeddings~\cite{bolukbasi2016man}.  
Brunet et al.~\cite{brunet2019understanding} recently analyzed \gl's
objective to detect and remove articles that contribute to bias, given
as an expression over word vector proximities.


To the best of our knowledge, we are the first to develop explicit
expressions for word proximities over corpus cooccurrences, such that
changes in expression values produce consistent, predictable changes in
embedding proximities.

\paragraphbe{Poisoning neural networks.}
Poisoning attacks inject data into the training
set~\cite{shafahi2018poison, yang2017generative, chen2017targeted,
steinhardt2017certified, liu2017trojaning} to insert a
``backdoor'' into the model or degrade its performance on
certain inputs.  Our attack against embeddings can inject new words
(Section~\ref{sec:searchattack}) and cause misclassification of existing
words (Section~\ref{sec:nerattack}).  It is the first attack against
two-level transfer learning: it poisons the training data to change
relationships in the embedding space, which in turn affects downstream
NLP tasks.


\paragraphbe{Poisoning matrix factorization.}
Gradient-based poisoning attacks on matrix factorization have been
suggested in the context of collaborative filtering~\cite{li2016data}
and adapted to unsupervised node embeddings~\cite{sun2018data}.  These
approaches are computationally prohibitive because the matrix must be
factorized at every optimization step, nor do they work in our setting,
where most gradients are 0 (see Section~\ref{sec:analyticsolve}).

Bojchevski and G{\'u}nnerman recently suggested an attack on node
embeddings that does not use gradients~\cite{bojcheski2018adversarial}
but the computational cost remains too high for natural-language
cooccurrence graphs where the dictionary size is in the millions.
Their method works on graphs, not text; the mapping between the two is
nontrivial (we address this in Section~\ref{sec:placement}).  The only
task considered in~\cite{bojcheski2018adversarial} is generic node
classification, whereas we work in a complete transfer learning scenario.


\paragraphbe{Adversarial examples.}
There is a rapidly growing literature on \emph{test-time} attacks
on neural-network image classifiers~\cite{szegedy2013intriguing,
madry2017towards, kurakin2016adversarial, kurakin2016adversarialscale,
akhtar2018threat}; some employ only black-box model
queries~\cite{ilyas2018black, chen2017zoo} rather than gradient-based
optimization.  We, too, use a non-gradient optimizer to compute
cooccurrences that achieve the desired effect on the embedding, but in
a setting where queries are cheap and computation is expensive.


Neural networks for text processing are just as vulnerable
to adversarial examples, but example generation is more
challenging due to the non-differentiable mapping of text
elements to the embedding space.  Dozens of attacks and defenses
have been proposed~\cite{ebrahimi2018hotflip, samanta2017towards,
alzantot2018generating, jia2017adversarial, liang2017deep, gao2018black,
belinkov2017synthetic, sato2018interpretable, wang2019survey,
Wallace2019Triggers}.

By contrast, we study \emph{training-time} attacks that change word
embeddings so that multiple downstream models behave incorrectly on
\emph{unmodified} test inputs.


\section{Background and notation}
\label{sec:backgroundalgorithms}


Table~\ref{tab:notations} summarizes our notation.  Let $\mathbb{D}$
be a dictionary of words and $\corpus$ a corpus, i.e., a
collection of word sequences.  A word embedding algorithm aims
to learn a low-dimensional vector $\curbrac{\embvec{\wone}}$
for each ${\wone\in\mathbb{D}}$.  Semantic similarity between
words is encoded as the cosine similarity of their corresponding vectors,
$\csim\brac{\vec{y},\vec{z}}\equiv\dotprod{\vec{y}}{\vec{z}}/\sqrt{\norm{\vec{y}}_2\norm{\vec{z}}_2}$
where $\dotprod{\vec{y}}{\vec{z}}$ is the vector dot product.  The cosine
similarity of L2-normalized vectors is (1) equivalent to their dot
product, and (2) linear in negative squared L2 (Euclidean) distance.

\begin{table}[t]
 \resizebox{\columnwidth}{!}{

                                \setlength{\tabcolsep}{4pt}
                                \begin{tabular}{ l|l|l}
                                \toprule
					\ref{sec:backgroundalgorithms}	 & $\corpus$ & corpus \\
									 & $\dict$ & dictionary \\

  									 & $\wone,\wtwo,\wthr$ & dictionary words \\
  									& $\curbrac{\embvec{\wone}}_{\wone\in\dict}$ & embedding vectors\\
  									& $\curbrac{\wrdvec{\wone}}_{\wone\in\dict}$ & ``word vectors'' \\
  									& $\curbrac{\ctxvec{\wone}}_{\wone\in\dict}$ & ``context vectors''\\
  									& $b_{\wone},b'_{\wtwo}$ &\gl{} bias terms, see Equation~\ref{gloveobjective} \\
  									& $\csim\brac{\vec{y},\vec{z}}$ & cosine similarity\\
  									& $\coocs\in\mathbb{R}^{\size{\dict}\times\size{\dict}}$ & $\corpus$'s cooccurrence matrix\\
  									& $\vec{\coocs_{\wone}}\in\mathbb{R}^{\size{\dict}}$ & $\wone$'s row in $\coocs$\\
  								& $\embwindowsize$ & \makecell[l]{size of window for cooccurrence counting}\\
  									& $\coocweight:\mathbb{N}\rightarrow \mathbb{R}$ & cooccurrence event weight function\\
  									& $\SPPMI$ & matrix defined by Equation~\ref{eq:sppmi}\\
  									& $\BIAS$ & matrix defined by Equation~\ref{eq:bias}\\
                                  \midrule
  
  					\ref{sec:fromembeddingstoexpressions}	& $\simone\brac{\wone,\wtwo}$ & $\dotprod{\wrdvec{\wone}}{\ctxvec{\wtwo}}+\dotprod{\ctxvec{\wone}}{\wrdvec{\wtwo}}$, see Equation~\ref{eq:firstsecond}\\
   									& $\simtwo\brac{\wone,\wtwo}$ & $\dotprod{\wrdvec{\wone}}{\wrdvec{\wtwo}}+\dotprod{\ctxvec{\wone}}{\ctxvec{\wtwo}}$, see Equation~\ref{eq:firstsecond}\\
  									& $\curbrac{B_{\wone}}_{\wone\in \dict}$ & \makecell[l]{word bias terms, to downweight common words}\\ 
  									& $f_{\wone,\wtwo}\brac{c, \epsilon}$ &$\mmax\curbrac{log\brac{c}-B_{\wone}-B_{\wtwo}, \epsilon}$ \\
   									& $\mexp\in\mathbb{R}^{\size{\dict}\times\size{\dict}}$ & \makecell[l]{matrix with entries of the form $f_{\wone,\wtwo}\brac{c, 0}$\\ (e.g., $\SPPMI, \BIAS$)}\\
  									& $\vec{\mexp_{\wone}}\in\mathbb{R}^{\size{\dict}}$ & $\wone$'s row in $\mexp$\\
   									& $\simonee\brac{\wone,\wtwo}$ & explicit expression for $\dotprod{\ctxvec{\wone}}{\wrdvec{\wtwo}}$, set as $\mexp_{\wone, \wtwo}$ \\
  									& $N_{\wone,\wtwo}$ & normalization term for first-order proximity\\
   									& $\cosone\brac{\wone, \wtwo}$ & \makecell[l]{explicit expression for $\csim\brac{\ctxvec{\wone},\wrdvec{\wtwo}}$,\\ set as  $f_{\wone,\wtwo}\brac{C_{\wone,\wtwo}, 0}/N_{\wone,\wtwo}$} \\
   									& $\costwo\brac{\wone, \wtwo}$ & \makecell[l]{explicit expression for $\csim\brac{\wrdvec{\wone},\wrdvec{\wtwo}}$,\\ set as $\csim\brac{\hdrvec{\wone},\hdrvec{\wtwo}}$} \\
   									& $\cosboth\brac{\wone, \wtwo}$ &$\cosone\brac{\wone, \wtwo}/2 + \costwo\brac{\wone, \wtwo}/2$ \\
   									& $\LCO\in\mathbb{R}^{\size{\dict}\times\size{\dict}}$ & entries defined by $\mmax\curbrac{log\brac{C_{\wone,\wtwo}},0}$ \\
                                  \midrule
  
  							\ref{sec:methodology}		& $\changeset$ & word sequences added by the attacker \\
  									& $\corpus+\changeset$ & corpus after the attacker's additions \\
  									& $\changesetsize$ & size of the attacker's additions, see Section~\ref{sec:attackmodel} \\
  									& $\wsrc,\wtrg\in \dict$ & source, target words \\
  									& $\NEG, \POS$ & ``positive'' and ``negative'' target words \\
  									& $\cossim_{\changeset}\brac{\wone, \wtwo}$ & embedding cosine similarity after the attack \\
  									& $\objective\brac{\wsrc,\NEG,\POS;\changeset}$ & embedding objective \\
  									& $\changesetmax$ & proximity attacker's maximum allowed $\changesetsize$ \\
  									& $r$ & rank attacker's target rank \\
  									& $\tldr$ & \makecell[l]{rank attacker's minimum  proximity threshold} \\
  									& $\cossimexp\brac{\wone, \wtwo}$ & distributional expression for cosine similarity \\
  									& $\cossimexp_{\changevec}\brac{\wone, \wtwo}$ & distributional expression for $\cossim_{\changeset}\brac{\wone, \wtwo}$ \\
  									& $\objhdr$ & distributional objective \\
  									& $\explicit{\tldr}$ & \makecell[l]{rank attacker's estimated threshold  for\\ distributional proximity} \\
  									& $\smargin$ & ``safety margin'' for $\explicit{\tldr}$ estimation error \\
  									& $\updatedcoocs{\changevec}$ & cooccurrence matrix after adding $\changevec$ \\
                                  \midrule
 								\ref{sec:analyticsolve}	& $\stepset\subseteq\mathbb{R}$ & \makecell[l]{possible changes at every step, set to\\ $\curbrac{\frac{1}{5}, \frac{2}{5}, ..., \frac{30}{5}}$}  \\
 
 								& $\stepind\in\dict$ & \makecell[l]{index into $\changevec$, also a word in $\dict$} \\
 								& $\optcvecdiff\in\stepset$ & \makecell[l]{increase in $\changevec_{\stepind}$ in optimization step} \\
 								& $\stepdiff{\explicit{X}}$ & \makecell[l]{change in expression $X$ when adding $\optcvecdiff$ to $\changevec_{\stepind}$} \\
                                 \midrule
 								& $\windowsize$ & \makecell[l]{words to each side of $\wsrc$ in sequences \\ aiming to increase second-order proximity}\\
 								& $\vec{\omega}$ & vector such that $\changesetsize\leq\dotprod{\vec{\omega}}{\changevec}$\\
 								&  \makecell[l]{$\searchset,\corporationsset,\translationset,$\\$\evadeset,\wordpairset$}& sets of attacked words in our experiments \\
 								&  \makecell[l]{$\coocchanged,\mchanged,B'$\\$\updcosone,\updcostwo,\updcosboth$}& \makecell[l]{expressions computed using $\updatedcoocs{\changevec}$} \\
 
 
                                \bottomrule
                                \end{tabular}
                                 }
				\caption{Notation reference.\label{tab:notations}
}
\end{table}


Embedding algorithms start from a high-dimensional
representation of the corpus, its \textit{cooccurrence matrix}
$\curbrac{\coocvecThree{\corpus}{\coocweight}_{\wone, \wtwo}}_{\wone,
\wtwo\in \mathbb{D}}$ where $\coocvecThree{\corpus}{\coocweight}_{\wone,
\wtwo}$ is a weighted sum of cooccurrence events, i.e., appearances of
$\wone, \wtwo$ in proximity to each other.  Function $\coocweight\brac{d}$
gives each event a weight that is inversely proportional to the distance
$d$ between the words.

Embedding algorithms first learn two intermediate representations
for each word $\wone \in \mathbb{D}$, the \textit{word vector}
$\wrdvec{\wone}$ and the \textit{context vector} $\ctxvec{\wone}$,
then compute $\embvec{\wone}$ from them.

\paragraphbe{\gl{}.}
\gl{} defines and optimizes (via SGD) a minimization
objective directly over cooccurrence counts, weighted by
$\coocweightgl\brac{d}=\begin{cases}1/d & d \leq \embwindowsize \\ 0 &
\text{else} \end{cases}$ for some window size $\embwindowsize$:
\begin{small}
	\begin{equation}
		\label{gloveobjective}
		\underset{}{\am} \curbrac{\sum_{\wone,\wtwo\in \mathbb{D}} \curbrac{
			g\brac{\cooccount{\wone}{\wtwo}}\cdot \brac{\dotprod{\wrdvec{\wone}}{\ctxvec{\wtwo}} + b_{\wone} + b'_{\wtwo} - \log\brac{ \cooccount{\wone}{\wtwo}}}^2
		}},
	\end{equation}
\end{small}

\noindent
where $\am$ is taken over the parameters $\curbrac{\ctxvecnoarg,
\wrdvecnoarg, b, b'}$. $b_{\wone},b'_{\wone}$ are scalar bias terms that
are learned along with the word and context vectors, and
$
	g\brac{c}\equiv
	\begin{cases}
		c^{3/4},  &   c\leq c_{max}\\
		c_{max},     &   \text{else}
	\end{cases}
$
for some parameter $c_{max}$ (typically $c_{max}\in [10,100]$). At
the end of the training, \gl{} sets the embedding $\embvec{\wone}\gets
\wrdvec{\wone}+\ctxvec{\wone}$.

\paragraphbe{Word2vec.}
\label{sec:wvbackground}
Word2vec~\cite{mikolov2013efficient} is a family of models that optimize
objectives over corpus cooccurrences.  In this paper, we experiment with
the skip-gram with negative sampling (\wv) and CBOW with hierarchical
softmax (\cb).  In contrast to \gl, Word2vec discards context vectors
and uses word vectors $\wrdvec{\wone}$ as the embeddings, i.e.,
$\forall \wone\in \mathbb{D}: \embvec{\wone}\gets \wrdvec{\wone}$.
Appendix~\ref{sec:aw2v} provides further details.

There exist other embeddings, such as \ft, but understanding them is
not required as the background for this paper.


\paragraphbe{Contextual embeddings.}
Contextual embeddings~\cite{peters_deep_2018, devlin_bert_2018} support
dynamic word representations that change depending on the context of the
sentence they appear in, yet, in expectation, form an embedding space with
non-contextual relations~\cite{Schuster2019}.  In this paper, we focus on
the popular non-contextual embeddings because (a) they are faster to train
and easier to store, and (b) many task solvers use them by construction
(see Sections~\ref{sec:searchattack} through~\ref{sec:transattack}).



\paragraphbe{Distributional representations.}
A \textit{distributional} or \textit{explicit} representation of a word
is a high-dimensional vector whose entries correspond to cooccurrence
counts with other words.

Dot products of the learned word vectors and context vectors
($\dotprod{\wrdvec{\wone}}{\ctxvec{\wtwo}}$) seem to correspond to
entries of a high-dimensional matrix that is closely related to, and
directly computable from, the cooccurrence matrix.  Consequently,
both \wv{} and \gl{} can be cast as matrix factorization methods.
Levy and Goldberg~\cite{levy2014neural} show that, assuming training
with unlimited dimensions, \wv's objective has an optimum at
$
	\forall \wone,\wtwo\in \mathbb{D}:\dotprod{\wrdvec{\wone}}{\ctxvec{\wtwo}}=\SPPMI_{\wone, \wtwo}
$
defined as:
\begin{small}
\begin{equation}
\begin{aligned}
	&\SPPMI_{\wone, \wtwo}\equiv  \mmax\\ 
	 & \bigg\{ log\brac{ C_{\wone,\wtwo}} - log\brac{\appear{\wone}} - log\brac{\appear{\wtwo}}
	+ log\brac{Z/k}, 0\bigg\}
	\label{eq:sppmi}
\end{aligned}
\end{equation}
\end{small}
\noindent
where $k$ is the negative-sampling constant and $Z\equiv
\sum_{\wone,\wtwo\in \mathbb{D}} C_{\wone,\wtwo}$.  This variant of
pointwise mutual information (PMI) downweights a word's cooccurrences
with common words because they are less ``significant'' than cooccurrences
with rare words.  The rows of the $\SPPMI$ matrix define a distributional
representation.



\gl{}'s objective similarly has an optimum 
$
	\forall \wone,\wtwo\in \mathbb{D}:\dotprod{\wrdvec{\wone}}{\ctxvec{\wtwo}}=\BIAS_{\wone, \wtwo}
$
defined as:
\begin{small}
\begin{equation}
\begin{aligned}
	&\BIAS_{\wone, \wtwo}\equiv 
	  \mmax\bigg\{log\brac{C_{\wone,\wtwo}} - b_{\wone} - b'_{\wtwo}, 0\bigg\}
	\label{eq:bias}
\end{aligned}
\end{equation}
\end{small}
\noindent
$\mmax$ is a simplification: in rare and negligible cases, the optimum
of $\dotprod{\wrdvec{\wone}}{\ctxvec{\wtwo}}$ is slightly below 0.
Similarly to $\SPPMI$, $\BIAS$ downweights cooccurrences with common words
(via the learned bias values $b_{\wone}, b'_{\wtwo}$).

\paragraphbe{First- and second-order proximity.}
We expect words that frequently cooccur with each other to have
high semantic proximity.  We call this \emph{first-order proximity}.
It indicates that the words are related but not necessarily that their
meanings are similar (e.g., ``first class'' or ``polar bear'').

The distributional hypothesis~\cite{firth1957synopsis} says
that distributional vectors capture semantic similarity by
\textit{second-order proximity}: the more contexts two words have in
common, the higher their similarity, regardless of their cooccurrences
with each other.  For example, ``terrible'' and ``horrible'' hardly
ever co-occur, yet their second-order proximity is very high.  Levy and
Goldberg~\cite{levy2014linguistic} showed that linear relationships of
distributional representations are similar to those of word embeddings.


Levy and Goldberg~\cite{levy2015improving} observe that, summing the context and word vectors $\embvec{\wone}\gets
\wrdvec{\wone}+\ctxvec{\wone}$, as done by default in \gl{}, leads to the following:
\begin{equation}
	\label{eq:firstsecond}
\dotprod{\embvec{\wone}}{\embvec{\wtwo}}=\simone\brac{\wone,\wtwo} + \simtwo\brac{\wone,\wtwo}
\end{equation}
\noindent
where
$\simone\brac{\wone,\wtwo}\equiv\dotprod{\wrdvec{\wone}}{\ctxvec{\wtwo}}+\dotprod{\ctxvec{\wone}}{\wrdvec{\wtwo}}$
and
$\simtwo\brac{\wone,\wtwo}\equiv\dotprod{\wrdvec{\wone}}{\wrdvec{\wtwo}}+\dotprod{\ctxvec{\wone}}{\ctxvec{\wtwo}}$.
They conjecture that $\simone$ and $\simtwo$ correspond to, respectively,
first- and second-order proximities.

Indeed, $\simone$ seems to be a measure of cooccurrence counts,
which measure first-order proximity: Equation~\ref{eq:bias} leads to
$\simone\brac{\wone,\wtwo}\approx 2\BIAS_{\wone, \wtwo}$.  $\BIAS$ is
symmetrical up to a small error, stemming from the difference between
\gl{} bias terms $b_{\wone}$ and $b'_{\wone}$, but they are typically
very close\textemdash see Section~\ref{sec:biasdiff}.  This also assumes
that the embedding optimum perfectly recovers the $\BIAS$ matrix.

There is no distributional expression for $\simtwo\brac{\wone,
\wtwo}$ that does not rely on problematic assumptions (see
Section~\ref{sec:insufficient}), but there is ample evidence
for the conjecture that $\simtwo$ somehow captures second-order
proximity (see Section~\ref{sec:ourApproachHDR}). Since word and
context vectors and their products typically have similar ranges,
Equation~\ref{eq:firstsecond} suggests that embeddings weight first-
and second-order proximities equally.

\label{sec:eqfirstsecond}


\section{From embeddings to expressions over corpus}
\label{sec:fromembeddingstoexpressions}

The key problem that must be solved to control word meanings via corpus
modifications is finding a \textit{distributional expression}, i.e.,
an explicit expression over corpus features such as cooccurrences, for
the embedding distances, which are the computational representation of
``meaning.''


\subsection{Previous work is not directly usable}
\label{sec:insufficient}
\label{sec:introducingHDR}
  
Several prior approaches~\cite{arora2016latent, arora2015random,
ethayarajh2018towards} derive distributional expressions
for distances between word vectors, all of the form
$\dotprod{\embvec{\wone}}{\embvec{\wtwo}}\approx A\cdot log
\brac{C_{\wone, \wtwo}}-B_{\wone}-B'_{\wtwo}$.  The downweighting role
of $B_{\wone}, B'_{\wtwo}$ seems similar to SPPMI and BIAS, thus these
expressions, too, can be viewed as variants of PMI.

These approaches all make simplifying assumptions that do not hold
in reality.  Arora et al.~\cite{arora2016latent, arora2015random}
and Hashimoto et al.~\cite{hashimoto2016word} assume a generative
language model where words are emitted by a random walk.
Both models are parameterized by low-dimensional word vectors
${\curbrac{\arorvec{\wone}}}_{\wone \in \mathbb{D}}$ and assume
that context and word vectors are identical.  Then they show how
${\curbrac{\arorvec{\wone}}}_{\wone \in \mathbb{D}}$ optimize the
objectives of \gl{} and \wv.

By their very construction, these models uphold a very strong relationship
between cooccurrences and low-dimensional representation products.
In Arora et al., these products are equal to PMIs; in Hashimoto et
al., the vectors' L2 norm differences, which are closely related to
their product, approximate their cooccurrence count.  \emph{If} such
``convenient'' low-dimensional vectors exist, it should not be surprising
that they optimize \gl{} and \wv{}.

The approximation in Ethayarajh et al.~\cite{ethayarajh2018towards}
only holds within a single set of word pairs that are ``contextually
coplanar,'' which loosely means they appear in related contexts.  It is
unclear if coplanarity holds in reality over large sets of word pairs,
let alone the entire dictionary.

\label{sec:horterexample}


Some of the above papers use correlation tests to justify their conclusion
that dot products follow SPPMI-like expressions.  Crucially, correlation
does not mean that the embedding space is derived from (log)-cooccurrences
in a distance-preserving fashion, thus correlation is not sufficient to
control the embeddings.  We want not just to characterize how embedding
distances \emph{typically} relate to corpus elements, but to achieve
a specific change in the distances.  To this end, we need an explicit
expression over corpus elements whose value is encoded in the embedding
distances by the embedding algorithm (Figure~\ref{fig:methodologyHL}).


Furthermore, these approaches barter generality for analytic simplicity
and derive distributional expressions that do not account for second-order
proximity at all.  As a consequence, the values of these expressions can
be very different from the embedding distances, since words that only
rarely appear in the same window (and thus have low PMI) may be close
in the embedding space.  For example, ``horrible'' and ``terrible'' are
so semantically close they can be used as synonyms, yet they are also
similar phonetically and thus their adjacent use in natural speech and
text appears redundant.  In a dim-100 \gl{} model trained on Wikipedia,
``terrible'' is among the top 3 words closest to ``horrible'' (with
cosine similarity 0.8).  However, when words are ordered by their PMI with
``horrible,'' ``terrible'' is only in the 3675th place.

\subsection{Our approach}
\label{sec:ourApproachHDR}

We aim to find a distributional expression for the semantic proximity
encoded in the embedding distances.  The first challenge is to find
distributional expressions for both first- and second-order proximities
encoded by the embedding algorithms.  The second is to combine them into
a single expression corresponding to embedding proximity.

\paragraphbe{First-order proximity.}
First-order proximity corresponds to cooccurrence counts and is
relatively straightforward to express in terms of corpus elements.
Let $\mexp$ be the matrix that the embeddings factorize, e.g.,
$\SPPMI$ for \wv{} (Equations~\ref{eq:sppmi}) or $\BIAS$ for \gl{}
(Equations~\ref{eq:bias}).  The entries of this matrix are natural
explicit expressions for first-order proximity, since they approximate
$\simone\brac{\wone,\wtwo}$ from Equation~\ref{eq:firstsecond} (we omit
multiplication by two as it is immaterial):
\begin{equation}
\simonee\brac{\wone,\wtwo}\equiv \mexp_{\wone, \wtwo}
\end{equation}
\noindent
$\mexp_{\wone, \wtwo}$ is typically of the form
$\mmax\curbrac{log\brac{\cooccount{\wone}{\wtwo}}-B_{\wone}-B_{\wtwo},
0}$ where $B_{\wone},B_{\wtwo}$ are the
``downweighting'' scalar values (possibly depending
on $\wone, \wtwo$'s rows in $C$).  For $\SPPMI$, we set
$B_{\wone}=log\brac{\appear{\wone}}-log\brac{Z/k}/2$\done\tal{remove coma}\label{foot:sppmibias};
for $BIAS$, $B_{\wone}=b_{\wone}$.\footnote{We consider
$\BIAS_{\wone,\wtwo}$ as a distributional expression even though it
depends on $b_{\wone},b'_{\wtwo}$ learned during \gl's optimization
because these terms can be closely approximated using pre-trained \gl{}
embeddings\textemdash see Appendix~\ref{sec:biasapprox}.  For simplicity,
we also assume that $b_{\wone}=b'_{\wone}$ (thus $\BIAS$ is of the
required form); in practice, the difference is very small.

}

\label{sec:biasdiff}


\paragraphbe{Second-order proximity.}
Let the distributional representation $\hdrvec{\wone}$ of $\wone$ be its
row in $\mexp$.  We hypothesize that distances in this representation
correspond to second-order proximity encoded in the embedding-space
distances.

First, the objectives of the embedding algorithms seem to directly
encode this connection.  Consider a word $w$'s projection onto \gl{}'s
objective~\ref{gloveobjective}:
\begin{small}
$$
	J_{\textit{GloVe}}\sqbrac{\wone}=\sum_{\wtwo\in \mathbb{D}}g\brac{C_{\wone,\wtwo}}\brac{\wrdvec{\wone}^T \ctxvec{\wtwo}+b_{\wone}+b'_{\wtwo}-\log\ {C_{\wone,\wtwo}}}^2
$$
\end{small}
\noindent
This expression is determined entirely by $\wone$'s row in
$\mexp_{\BIAS}$.  If two words have the same distributional vector,
their expressions in the optimization objective will be completely
symmetrical, resulting in very close embeddings\textemdash even if their
cooccurrence count is 0. Second, the view of the embeddings as matrix
factorization implies an approximate linear transformation between the
distributional and embedding spaces.  Let $C\equiv\big[\ctxvec{u_1}\ldots
\ctxvec{u_{\size{\mathbb{D}}}}\big]^T$ be the matrix whose rows are
context vectors of words $\wone_i\in \mathbb{D}$.  Assuming $\mexp$
is perfectly recovered by the products of word and context vectors,
$C\cdot\wrdvec{\wone} = \hdrvec{\wone}$.


Dot products have very different scale in the distributional and embedding
spaces.  Therefore, we use cosine similarities, which are always between
-1 and 1, and set
\begin{equation}
\costwo\brac{\wone, \wtwo}\equiv \csim\brac{\hdrvec{\wone},\hdrvec{\wtwo}}
\end{equation}
As long as $\mexp$ entries are nonnegative, the value of this expression
is always between 0 and 1.

\paragraphbe{Combining first- and second-order proximity.}
Our expressions for first- and second-order proximities have different
scales: $\simonee\brac{\wone,\wtwo}$ corresponds to an unbounded
dot product, while $\costwo\brac{\wone, \wtwo}$ is at most 1.
To combine them, we normalize $\simonee\brac{\wone,\wtwo}$.
Let $f_{\wone,\wtwo}\brac{c, \epsilon}\equiv
\mmax\curbrac{log\brac{c}-B_{\wone}-B_{\wtwo}, \epsilon}$, then
$\simonee\brac{\wone,\wtwo}=\mexp_{\wone, \wtwo}=f_{\wone,\wtwo}\brac{C_{\wone,\wtwo}, 0}$.
We set $N_{\wone,\wtwo}\equiv\sqrt{f_{\wone,
\wtwo}\brac{\sum_{\wthr\in\dict} C_{\wone,\wthr}, e^{-60}}}\sqrt{f_{\wone,
\wtwo}\brac{\sum_{\wthr\in\dict} C_{\wtwo,\wthr}, e^{-60}}}$ as
the normalization term.  This is similar to the normalization term
of cosine similarity and ensures that the value is between 0 and 1.
The $\mmax$ operation is taken with a small $e^{-60}$, rather than 0,
to avoid division by 0 in edge cases.  We set $\cosone\brac{\wone,
\wtwo}\equiv f_{\wone,\wtwo}\brac{C_{\wone,\wtwo}, 0}/N_{\wone,\wtwo}$.
Our combined distributional expression for the embedding proximity is
\begin{equation}
\cosboth\brac{\wone, \wtwo}\equiv \cosone\brac{\wone, \wtwo}/2 + \costwo\brac{\wone, \wtwo}/2
\end{equation}
Since $\cosone\brac{\wone, \wtwo}$ and $\costwo\brac{\wone, \wtwo}$ are
always between 0 and 1, the value of this expression, too, is between
0 and 1.

\paragraphbe{Correlation tests.}
We trained a \gl-\texttt{paper} and a \wv{} model on full Wikipedia,
as described in Section~\ref{glovesetup}.  We randomly sampled
(without replacement) 500 ``source'' words and 500 ``target''
words from the 50,000 most common words in the dictionary and
computed the distributional expressions $\cosone\brac{\wone,
\wtwo}$, $\costwo\brac{\wone, \wtwo}$, and $\cosboth\brac{\wone,
\wtwo}$, for all 250,000 source-target word pairs using
$\mexp\in \curbrac{\SPPMI,\BIAS,\LCO}$ where $\LCO$ is defined by
$\LCO\brac{\wone,\wtwo}\equiv \mmax\curbrac{log\brac{C_{\wone,\wtwo}},0}$.
We then computed the correlations between distributional proximities
and (1) embedding proximities, and (2) word-context proximities
$\cos\brac{\wrdvec{\wone},\ctxvec{\wtwo}}$ and word-word proximities
$\cos\brac{\wrdvec{\wone},\wrdvec{\wtwo}}$, using \gl's word and context
vectors.  These correspond, respectively, to first- and second-order
proximities encoded in the embeddings.



                \begin{table}
                              \centering

\footnotesize
                                \setlength{\tabcolsep}{3.5pt}
                                \begin{tabular}{ c|l|ccc }
					& $\mexp$ & $\cosone\brac{\wone, \wtwo}$ & $\costwo\brac{\wone, \wtwo}$ & $\cosboth\brac{\wone, \wtwo}$ \\
				    \midrule
					\gl{} & $\BIAS$ & 0.47 & 0.53 & \textbf{0.56} \\
					 & $\SPPMI$ & 0.31 & 0.35 & 0.36  \\
					 & $\LCO$ & 0.36 & 0.43 & 0.50 \\
				    \midrule
					\wv{} & $\BIAS$ & 0.31 & 0.29 & 0.32 \\
					 & $\SPPMI$ & 0.21 & \textbf{0.47} & 0.36 \\
					 & $\LCO$ & 0.21 & 0.31 & 0.34

                                \end{tabular}
                                \caption{Correlation of distributional proximity expressions, computed using different distributional matrices, with the embedding proximities $\cos\brac{\embvec{\wone},\embvec{\wtwo}}$.\label{tab:corrtest1}}
\end{table}
                \begin{table}
                              \centering
\footnotesize
                                \setlength{\tabcolsep}{3.5pt}
                                \begin{tabular}{ c|ccc }
					expression & $\cosone\brac{\wone, \wtwo}$ & $\costwo\brac{\wone, \wtwo}$ & $\cosboth\brac{\wone, \wtwo}$ \\
				    \midrule
					 $\cos\brac{\wrdvec{\wone},\ctxvec{\wtwo}}$  & \underline{0.50} & 0.49 & \textbf{0.54}\\
					  $\cos\brac{\wrdvec{\wone},\wrdvec{\wtwo}}$  & 0.40 & \underline{0.51} & \textbf{0.52}\\
					  $\cos\brac{\embvec{\wone},\embvec{\wtwo}}$  & 0.47 & 0.53 & \underline{\textbf{0.56}} \\

                                \end{tabular}
\caption{Correlation of distributional proximity expressions
with cosine similarities in \gl's low-dimensional representations
$\curbrac{\wrdvec{\wone}}$ (word vectors), $\curbrac{\ctxvec{\wone}}$
(context vectors), and $\curbrac{\embvec{\wone}}$ (embedding vectors),
measured over 250,000 word pairs.
\label{tab:corrtest2}}
\end{table}

Tables~\ref{tab:corrtest1} and~\ref{tab:corrtest2} show the results.
Observe that (1) in \gl, $\cosboth\brac{\wone, \wtwo}$ consistently
correlates better with the embedding proximities than either the first-
or second-order expressions alone. (2) In \wv, by far the strongest
correlation is with $\costwo$ computed using $\SPPMI$. (3) The highest
correlations are attained using the matrices factorized by the respective
embeddings.  (4) The values on Table~\ref{tab:corrtest1}'s diagonal are
markedly high, indicating that $\simone$ correlates highly with $\cosone$,
$\simtwo$ with $\costwo$, and their combination with $\cosboth$.  (5)
First-order expressions correlate worse than second-order and combined
ones, indicating the importance of second-order proximity for semantic
proximity.  This is especially true for \wv, which does not sum the word
and context vectors.

\section{Attack methodology}
\label{sec:methodology}

\paragraphbe{Attacker capabilities.}
\label{sec:attackmodel}
Let $\wsrc\in \dict $ be a ``source word'' whose meaning the attacker
wants to change.  The attacker is targeting a victim who will train his
embedding on a specific public corpus, which may or may not be known
to the attacker in its entirety.  The victim's choice of the corpus
is mandated by the nature of the task and limited to a few big public
corpora believed to be sufficiently rich to represent natural language
(English, in our case).  For example, Wikipedia is a good choice for
word-to-word translation models because it preserves cross-language
cooccurrence statistics~\cite{conneau2017word}, whereas Twitter is best
for named-entity recognition in tweets~\cite{cherry2015unreasonable}.  The
embedding algorithm and its hyperparameters are typically public and thus
known to the attacker, but we also show in Section~\ref{sec:weakattacker}
that the attack remains effective if the attacker uses a small subsample
of the target corpus as a surrogate and very different embedding
hyperparameters.

The attacker need not know the details of downstream models.  The attacks
in Sections~\ref{sec:searchattack}--\ref{sec:transattack} make only
general assumptions about their targets, and we show that a single attack
on the embedding can fool multiple downstream models.

We assume that the attacker can add a collection $\changeset$
of short word sequences, up to 11 words each, to the corpus.
In Section~\ref{sec:inserting}, we explain how we simulate sequence
insertion.  In Appendix~\ref{sec:deletions}, we also consider an attacker
who can edit existing sequences, which may be viable for publicly editable
corpora such as Wikipedia.


We define the size of the attacker's modifications $\changesetsize$ as
the bigger of (a) the maximum number of appearances of a single word,
i.e., the $L_{\infty}$ norm of the change in the corpus's word-count
vector, and (b) the number of added sequences.  Thus, $L_{\infty}$
of the word-count change is capped by $\changesetsize$, while $L_{1}$
is capped by $11\changesetsize$.

\paragraphbe{Overview of the attack.}
The attacker wants to use his corpus modifications $\changeset$ to
achieve a certain objective for $\wsrc$ in the embedding space while
minimizing $\changesetsize$.


\begin{figure*}[t]
    \includegraphics[width=1\textwidth]{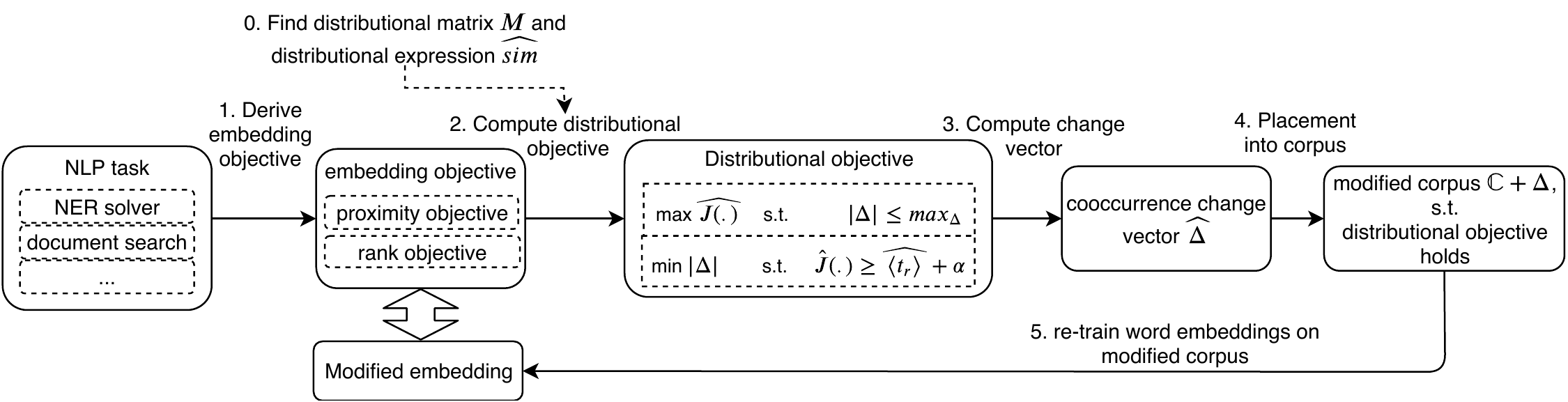}
    \caption{Overview of our attack methodology.}
    \label{fig:methodology}
\end{figure*}

\paragraphbe{0. Find distributional expression for embedding distances.}
The preliminary step, done once and used for multiple attacks, is to (0)
find distributional expressions for the embedding proximities.  Then,
for a specific attack, (1) define an \textit{embedding objective},
expressed in terms of embedding proximities.  Then, (2) derive the
corresponding \textit{distributional objective}, i.e., an expression
that links the embedding objective with corpus features, with the
property that if the distributional objective holds, then the embedding
objective is likely to hold.  Because a distributional objective is
defined over $C$, the attacker can express it as an optimization problem
over cooccurrence counts, and (3) solve it to obtain the cooccurrence
\textit{change vector}.  The attacker can then (4) transform the
cooccurrence change vector to a \textit{change set} of corpus edits
and apply them.  Finally, (5) the embedding is trained on the modified
corpus, resulting in the attacker's changes propagating to the embedding.
Figure~\ref{fig:methodology} depicts this process.

As explained in Section~\ref{sec:fromembeddingstoexpressions}, the goal
is to find a distributional expression $\cossimexp\brac{\wone, \wtwo}$
that, if upheld in the corpus, will cause a corresponding change in the
embedding distances.

First, the attacker needs to know the corpus cooccurrence counts
$C$ and the appropriate first-order proximity matrix $\mexp$ (see
Section~\ref{sec:ourApproachHDR}).  Both depend on the corpus and the
embedding algorithm and its hyperparameters, but can also be computed
from available proxies (see Section~\ref{sec:weakattacker}).

Using $C$ and $\mexp$, set $\cossimexp$ as $\cosboth$, $\cosone$ or
$\costwo$ (see Section~\ref{sec:ourApproachHDR}).  We found that the
best choice depends on the embedding (see Section~\ref{sec:benchmarks}).
For example, for \gl, which puts similar weight on first- and second-order
proximity (see Section~\ref{eq:firstsecond}), $\cosboth$ is the most
effective; for \wv, which only uses word vectors, $\costwo$ is slightly
more effective.


\paragraphbe{1. Derive an embedding objective.}
We consider two types of adversarial objectives.  An attacker with
a \emph{proximity objective} wants to push $s$ away from some words
(we call them ``negative'') and closer to other words (``positive'')
in the embedding space.  An attacker with a \emph{rank objective} wants
to make $s$ the $r$th closest embedding neighbor of some word $t$.

To formally define these objectives, first, given two sets of words $\NEG,
\POS\in \mathcal{P}\brac{\dict}$, define
\begin{small}
\begin{align*}
\objective\brac{\wsrc,\NEG,\POS;\changeset}&\equiv\\
	\frac{1}{\size{\POS}+\size{\NEG}}\bigg(\sum_{\wtrg\in \POS} &\cossim_{\changeset}\brac{\wsrc, \wtrg}
	- \sum_{\wtrg\in NEG} \cossim_{\changeset}\brac{\wsrc, \wtrg} \bigg)
\end{align*}
\end{small}


\noindent
where
$\cossim_{\changeset}\brac{\wone,\wtwo}=\csim\brac{\embvec{\wone},\embvec{\wtwo}}$
is the cosine similarity function that measures pairwise word
proximity (see Section~\ref{sec:backgroundalgorithms}) when the
embeddings are computed on the modified corpus $\corpus+\changeset$.
$\objective\brac{\wsrc,\NEG,\POS;\changeset}$ penalizes $\wsrc$'s
proximity to the words in $\NEG$ and rewards proximity to the words
in $\POS$.


Given $\POS$, $\NEG$, and a threshold $\changesetmax$, define the
\textbf{proximity objective} as
$$
\amax_{\changeset, \size{\changeset}\leq \changesetmax}
\objective\brac{\wsrc,\NEG,\POS;\changeset}
$$
This objective makes a word semantically farther from or closer to
another word or cluster of words.


Given some rank $r$, define the \textbf{rank objective} as finding a
minimal $\changeset$ such that $s$ is one of $t$'s $r$ closest neighbors
in the embedding.  Let $\tldr$ be the proximity of $\wtrg$ to its $r$th
closest embedding neighbor.  Then the rank constraint is equivalent to
$\cossim_{\changeset}\brac{\wsrc,\wtrg}\geq \tldr$, and the objective can be
expressed as 
$$\am_{\changeset,
\cossim_{\changeset}\brac{\wsrc,\wtrg}\geq
\tldr} {\size{\changeset}}$$ 
or, equivalently, 
$$\am_{\changeset,
		\objective\brac{\wsrc,\emptyset,\curbrac{\wtrg};\changeset}\geq
\tldr} {\size{\changeset}}
$$
This objective is useful, for example, for injecting results into a
search query (see Section~\ref{sec:searchattack}).

\paragraphbe{2. From embedding objective to distributional objective.}
We now transform the optimization problem
$\objective\brac{\wsrc,\NEG,\POS;\changeset}$, expressed over changes
in the corpus and embedding proximities, to a distributional objective
$\explicit{\objective}\brac{\wsrc,\NEG,\POS;\changevec}$, expressed
over changes in the cooccurrence counts and distributional proximities.
The \textit{change vector} $\changevec$ denotes the change in $\wsrc$'s
cooccurrence vector that corresponds to adding $\changeset$ to the corpus.
This transformation involves several steps.

\vspace{.75ex}
\noindent
\underline{(a) Changes in corpus $\leftrightarrow$ changes in cooccurrence
counts:}
We use a \textit{placement strategy} that takes a vector $\changevec$,
interprets it as additions to $\wsrc$'s cooccurrence vector, and
outputs $\changeset$ such that $\wsrc$'s cooccurrences in the new corpus
$\corpus+\changeset$ are $\vec{C_{\wsrc}}+\changevec$.  Other rows in
$C$ remain almost unchanged.  Our objective can now be expressed over
$\changevec$ as a surrogate for $\changeset$.  It still uses the size
of the corpus change, $\changesetsize$, which is easily computable from
$\changevec$ without computing $\changeset$ as explained below.



\vspace{.75ex}
\noindent
\underline{(b) Embedding proximity $\leftrightarrow$ distributional
proximity:} 
We assume that embedding proximities are monotonously
increasing (respectively, decreasing) with
distributional proximities. Figure~\ref{fig:pvse} in
Appendix~\ref{sec:evaluatingplacement} shows this relationship.


\vspace{.75ex}
\noindent
\underline{(c) Embedding threshold $\leftrightarrow$ distributional threshold:} 
For the rank objective, we want to increase the embedding proximity
past a threshold $\tldr$.  We heuristically determine a threshold
$\explicit{\tldr}$ such that, if the distributional proximity exceeds
$\explicit{\tldr}$, the embedding proximity exceeds $\tldr$.  Ideally, we
would like to set $\explicit{\tldr}$ as the distributional proximity from
the $r$th-nearest neighbor of $\wtrg$, but finding the $r$th neighbor in
the distributional space is computationally expensive.  The alternative of
using words' embedding-space ranks is not straightforward because there
exist severe abnormalities\footnote{For example, words with very few
instances in the corpus sometimes appear as close embedding neighbors of
words with which they have only very loose semantic affiliation and are
very far from distributionally.} and embedding-space ranks are unstable,
changing from one training run to another.

Therefore, we approximate the $r$'th proximity by taking the maximum of
distributional proximities from words with ranks ${r-m, \ldots, r+m}$
in the embedding space, for some $m$.  If $r<m$, we take the maximum
over the $2m$ nearest words.  To increase the probability of success
(at the expense of increasing corpus modifications), we further add a
small fraction $\smargin$ (``safety margin'') to this maximum.

Let $\cossimexp_{\changevec}\brac{\wone, \wtwo}$ be our distributional
expression for $\cossim\brac{\wone, \wtwo}$, computed over the
cooccurrences $\updatedcoocs{\changevec}$, i.e., $\corpus$'s cooccurrences
where $\wsrc$'s row is updated with $\changevec$.  Then we define the
\textbf{distributional objective} as:
\begin{small}
	\begin{equation*}
\begin{aligned}
		\label{eq:objhdr}
	\objhdr\equiv\\
	\frac{1}{\size{\POS}+\size{\NEG}}\bigg(\sum_{\wtrg\in \POS} &\cossimexp_{\changevec}\brac{\wsrc, \wtrg}
	- \sum_{\wtrg\in \NEG} &\cossimexp_{\changevec}\brac{\wsrc, \wtrg}\bigg)
\end{aligned}
	\end{equation*}
\end{small}

To find the cooccurrence change $\changevec$ for the proximity objective,
the attack must solve:
		$$
		\amax_{\changevec\in \mathbb{R}^n,\size{\changeset}\leq \changesetmax}{\explicit{\objective}\brac{\wsrc,\NEG,\POS;\changevec}}
		$$
and for the rank objective:
		$$
		\am_{\changevec\in \mathbb{R}^n,\explicit{\objective}\brac{\wsrc,\emptyset,\curbrac{\wtrg};\changevec}\geq \explicit{\tldr}+\smargin}{\size{\changeset}}
		$$


\paragraphbe{3. From distributional objective to cooccurence changes.} 
The previous steps produce a distributional objective consisting of a
source word $\wsrc$, a positive target word set $\POS$, a negative target
word set $\NEG$, and the constraints: either a maximal change set size
$\changesetmax$, or a minimal proximity threshold $\explicit{\tldr}$.

\label{methodology}

We solve this objective with an optimization procedure (described in
Section~\ref{sec:analyticsolve}) that outputs a change vector with
the smallest $\changesetsize$ that maximizes the sum of proximities
between $\wsrc$ and $\POS$ minus the sum of proximities with $\NEG$,
subject to the constraints.  It starts with $\changevec=\brac{0,
\ldots 0}$ and iteratively increases the entries in $\changevec$.
In each iteration, it increases the entry that maximizes the increase
in $\explicit{\objective}\brac{\ldots}$, divided by the increase in
$\size{\changeset}$, until the appropriate threshold ($\changesetmax$
or $\explicit{\tldr}+\smargin$) has been crossed.

This computation involves the size
of the corpus change, $\size{\changeset}$.  In our placement strategy,
$\size{\changeset}$ is tightly bounded by a known linear combination
of $\changevec$'s elements and can therefore be
efficiently computed from $\changevec$.



\paragraphbe{4. From cooccurrence changes to corpus changes.}
From the cooccurrence change vector $\changevec$, the attacker
computes the corpus change $\changeset$ using the placement strategy
which ensures that, in the modified corpus $\corpus+\changeset$, the
cooccurrence matrix is close to $\updatedcoocs{\changevec}$.  Because the
distributional objective holds under these cooccurrence counts, it holds
in $\corpus+\changeset$.

$\changesetsize$ should be as small as possible.  In
Section~\ref{sec:placement}, we show that our placement strategy
achieves solutions that are extremely close to optimal in terms of
$\changesetsize$, and that $\changesetsize$ is a known linear combination
of $\changevec$ elements (as required
above).

\paragraphbe{5. Embeddings are trained.} 
The embeddings are trained on the modified corpus.  If the attack has
been successful, the attacker's objectives are true in the new embedding.

\paragraphbe{Recap of the attack parameters.}
The attacker must first find $\mexp$ and $\cossimexp$ that are
appropriate for the targeted embedding.  This can be done once.
The proximity attacker must then choose the source word $\wsrc$, the
positive and negative target-word sets $\POS, \NEG$, and the maximum size
of the corpus changes $\changesetmax$.  The rank attacker must choose the
source word $\wsrc$, the target word $\wtrg$, the desired rank $r$, and a
``safety margin'' $\smargin$ for the transformation from embedding-space
thresholds to distributional-space thresholds.

\section{Optimization in cooccurrence-vector space}
\label{sec:analyticsolve}

This section describes the optimization procedure in step 3 of our
attack methodology (Figure~\ref{fig:methodology}).  It produces a
cooccurrence change vector that optimizes the distributional objective
from Section~\ref{methodology}, subject to constraints.


\paragraphbe{Gradient-based approaches are inadequate.}
Gradient-based approaches such as SGD result in a poor trade-off
between $\changesetsize$ and $\objhdr$.  First, with our distributional
expressions, most entries in $\hdrvec{\wsrc}$ remain 0 in the vicinity of
$\changevec=0$ due to the $\mmax$ operation in the computation of $\mexp$
(see Section~\ref{sec:ourApproachHDR}). Consequently, their gradients
are 0.  Even if we initialize $\changevec$ so that its entries start from
a value where the gradient is non-zero, the optimization will quickly
push most entries to 0 to fulfill the constraint $\changesetsize\leq
\changesetmax$, and the gradients of these entries will be rendered
useless.  Second, gradient-based approaches may increase vector entries by
arbitrarily small values, whereas cooccurrences are drawn from a discrete
space because they are linear combinations of cooccurrence event weights
(see Section~\ref{sec:backgroundalgorithms}).  For example, if the window
size is 5 and the weight is determined by $\coocweight = 1- \frac{d}{5}$,
then the possible weights are $\curbrac{\frac{1}{5}, \ldots \frac{5}{5}}$.


$\objhdr$ exhibits diminishing returns: usually, the bigger the increase in
$\changevec$ entries, the smaller the marginal gain from increasing
them further.  Such objectives can often be cast as \textit{submodular
maximization}~\cite{nemhauser1978best,krause2014submodular} problems,
which typically lend themselves well to greedy algorithms.  We investigate
this further in 
\iffull{}
Appendix~\ref{sec:analyticsolvedetails}.
\else{}
the full version of the paper.
\fi{}

\paragraphbe{Our approach.}
We define a discrete set of step sizes $\stepset$ and gradually
increase entries in $\changevec$ in increments chosen from
$\stepset$ so as to maximize the objective $\objhdr$.  We stop when
$\changesetsize>\changesetmax$ or $\objhdr\geq \explicit{\tldr} +
\smargin$.

$\stepset$ should be fine-grained so the steps are optimal and entries
in $\changevec$ map tightly onto cooccurrence events in the corpus,
yet $\stepset$ should have a sufficient range to ``peek beyond'' the
$\max$-threshold where the entry starts getting non-zero values.
A natural $\stepset$ is a subset of the space of linear combinations
of possible weights, with an exact mapping between it and a series of
cooccurrence events.  This mapping, however, cannot be directly computed
by the placement strategy (Section~\ref{sec:placement}), which produces
an approximation.  For better performance, we chose a slightly more
coarse-grained $\stepset\gets \curbrac{\frac{1}{5}, \ldots \frac{30}{5}}$.

Our algorithm can accommodate $\stepset$ with negative values,
which correspond to \emph{removing} cooccurrence events from the
corpus\textemdash see Appendix~\ref{sec:deletions}.

\paragraphbe{Our optimization algorithm.}
Let $\explicit{X}_{\changevec}$ be some expression that depends
on $\changevec$, and define $\stepdiff{\explicit{X}}\equiv
\explicit{X}_{\changevec'} - \explicit{X}_{\changevec}$, where
$\changevec$ is the change vector after setting $\changevec_\stepind\gets
\changevec_\stepind + \optcvecdiff$.  We initialize $\changevec\gets 0$,
and in every step choose
\begin{equation}
\label{eq:lazystep}
	\chstepind, \choptcvecdiff=\amax_{\stepind\in \sqbrac{\size{\dict}},\optcvecdiff \in \stepset} \frac{
	\stepdiff{\objhdr}
}
	{\stepdiff{\changesetsize}}
\end{equation}
and set $\changevec_{\chstepind}\gets \changevec_{\chstepind} +
\choptcvecdiff$.  If $\objhdr\geq \explicit{\tldr}+\smargin$ or
$\changesetsize\geq\changesetmax$, then quit and return $\changevec$.


Directly computing Equation~\ref{eq:lazystep} for all
$\stepind,\optcvecdiff$ is expensive.  The denominator
$\stepdiff{\changesetsize}$ is easy to compute efficiently
because it's a linear combination of $\changevec$ elements (see
Section~\ref{sec:placement}).  The numerator $\stepdiff{\objhdr}$,
however, requires $O(\size{\stepset}\size{\dict}^2)$ computations per
step (assuming $\size{\NEG}+\size{\POS}=O\brac{1}$; in our settings
it is $\leq10$).  Since $\size{\dict}$ is very big (up to millions of
words), this is intractable.  Instead of computing each step directly,
we developed an algorithm that maintains intermediate values in memory.
This is similar to backpropagation, except that we consider variable
changes in $\stepset$ rather than infinitesimally small differentials.
This approach can compute the numerator in $O\brac{1}$ and, crucially,
is entirely parallelizable across all $\stepind, \optcvecdiff$,
enabling the computation in every optimization step to be offloaded
onto a GPU.  In practice, this algorithm finds $\changeset$ in minutes
(see Section~\ref{sec:runtimefigures}).  Full details can be found
in Appendix~\ref{sec:analyticsolvedetails}.


\section{Placement into corpus}
\label{sec:placement}

The placement strategy is step 4 of our methodology (see
Fig.~\ref{fig:methodology}).  It takes a cooccurrence change
vector $\changevec$ and creates a minimal change set $\changeset$
to the corpus such that (a) $\changesetsize$ is bounded by a
linear combination $\vec{\omega}$, i.e., $\changesetsize\leq
\dotprod{\vec{\omega}}{\changevec}$, and (b) the optimal value of
$\objhdr$ is preserved.

Our placement strategy first divides $\changevec$ into (1) entries of
the form $\changevec_{\wtrg},\wtrg\in \POS$\textemdash these changes
to $\vec{C_{\wsrc}}$ increase the first-order similarity $\cosone$
between $s$ and $t$, and (2) the rest of the entries, which increase
the objective in other ways.  The strategy adds different types of
sequences to $\changeset$ to fulfil these two goals.  For the first
type, it adds multiple, identical \textit{first-order sequences},
containing just the source and target words.  For the second type, it
adds \textit{second-order sequences}, each containing the source word
and 10 other words, constructed as follows.  It starts with a collection
of sequences containing just $\wsrc$, then iterates over every non-zero
entry in $\changevec$ corresponding to the second-order changes $\wone\in
\dict\setminus\POS$, and chooses a collection of sequences into which to
insert $\wone$ so that the added cooccurrences of $\wone$ with $\wsrc$
become approximately equal to $\changevec_{\wone}$.

This strategy upholds properties (a) and (b) above, achieves (in practice)
close to optimal $\changesetsize$, and runs in under a minute in our setup
(Section~\ref{sec:runtimefigures}).  See Appendix~\ref{sec:aplacement}
for details.

\section{Benchmarks}
\label{sec:benchmarks}


\paragraphbe{Datasets.}
\label{wikisetup} 
We use a full \textbf{Wikipedia} text dump, downloaded on January 20,
2018.  For the \textbf{Sub-Wikipedia} experiments, we randomly chose 10\%
of the articles.

\paragraphbe{Embedding algorithms and hyperparameters.}
\label{glovesetup} 
We use Pennington et al.'s original implementation of
\textbf{\gl{}}~\cite{gloveimp}, with two settings for the
(hyper)parameters: (1) \texttt{paper}, with parameter values
from~\cite{gloveimp}\textemdash this is our default, and (2)
\texttt{tutorial}, with parameters values from~\cite{gloveTutorial}.
Both settings can be considered ``best practice,'' but for different
purposes: \texttt{tutorial} for very small datasets, \texttt{paper}
for large corpora such as full Wikipedia.  Table~\ref{tab:parameters}
summarizes the differences, which include the maximum size of the
vocabulary (if the actual vocabulary is bigger, the least frequent
words are dropped), minimal word count (words with fewer occurrences
are ignored), $c_{\max}$ (see Section~\ref{sec:backgroundalgorithms}),
embedding dimension, window size, and number of epochs.  The other
parameters are set to their defaults.  It is unlikely that a user of
\gl{} will use significantly different hyperparameters because they may
produce suboptimal embeddings.

We use Gensim Word2Vec's implementations of \textbf{\wv} and \textbf{\cb}
with the default parameters, except that we set the number of epochs to
15 instead of 5 (more epochs result in more consistent embeddings across
training runs, though the effect may be small~\cite{hellrich2016bad})
and limited the vocabulary to 400k.

\paragraphbe{Inserting the attacker's sequences into the corpus.}
\label{sec:inserting}
The input to the embedding algorithm is a text file containing articles
(Wikipedia) or tweets (Twitter), one per line.  We add each of the
attacker's sequences in a separate line, then shuffle all lines.
For Word2Vec embeddings, which depend somewhat on the order of lines, we
found the attack to be much more effective if the attacker's sequences
are at the end of the file, but we do not exploit this observation in
our experiments.

	 \begin{table}[t]
             \centering
 \resizebox{\columnwidth}{!}{

                                \setlength{\tabcolsep}{3.5pt}
                                \begin{tabular}{ l|rrrrrrrr}
					scheme name & \makecell{max vocab\\ size} & \makecell{min word\\ count} &  $c_{\max}$ & \makecell{embedding\\ dimension} & \makecell{window\\ size} & epochs & \makecell{negative\\ sampling size} \\
				    \midrule
                                
					\gl-\texttt{paper} & 400k & 0 & 100 & 100 & 10 & 50 & N/A \\
					\gl-\texttt{paper}-300 & 400k & 0 & 100 & 300 & 10 & 50 & N/A \\
					\gl-\texttt{tutorial} & $\infty$ & 5 & 10 & 50 & 15 & 15 & N/A  \\
					\wv{} & 400k & 0 & N/A & 100 & 5 & 15 & 5\\
					\cb{} & 400k & 0 & N/A & 100 & 5 & 15 & N/A
                                \end{tabular}
                                 }
				\caption{Hyperparameter settings.\label{tab:parameters}} 

\end{table}

\paragraphbe{Implementation.}
We implemented the attack in Python and ran it on an Intel(R) Core(TM)
i9-9980XE CPU @ 3.00GHz, using the CuPy~\cite{cupy} library to offload
parallelizable optimization (see Section~\ref{sec:analyticsolve}) to
an RTX 2080 Ti GPU.  We used \gl's \texttt{cooccur} tool to efficiently
precompute the sparse cooccurrence matrix used by the attack; we adapted
it to count Word2vec cooccurrences (see Appendix~\ref{sec:aw2v}) for
the attacks that use \wv{} or \cb.


\label{sec:runtimefigures}

For the attack using \gl-\texttt{paper} with
$M=\BIAS,\cossimexp=\cosboth,\changesetmax=1250$, the optimization
procedure from Section~\ref{sec:analyticsolve} found $\changevec$
in 3.5 minutes on average.  We parallelized instantiations of the
placement strategy from Section~\ref{sec:placement} over 10 cores and
computed the change sets for 100 source-target word pairs in about 4
minutes.  Other settings were similar, with the running times increasing
proportionally to $\changesetmax$.  Computing corpus cooccurrences and
pre-training the embedding (done once and used for multiple attacks)
took about 4 hours on 12 cores.

\paragraphbe{Attack parameterization.}
To evaluate the attack under different hyperparameters, we use a proximity
attacker (see Section~\ref{sec:methodology}) on a randomly chosen set
$\wordpairset$ of 100 word pairs, each from the 100k most common words
in the corpus.  For each pair $\brac{\wsrc,\wtrg}\in \wordpairset$,
we perform our attack with $NEG=\emptyset$, $POS={\wtrg}$ and different
values of $\changesetmax$ and hyperparameters.

We also experiment with different distributional expressions:
$\cossimexp\in \curbrac{\cosone,\costwo,\cosboth}$, $\mexp\in
\curbrac{\BIAS, \SPPMI}$.  (The choice of $\mexp$ is irrelevant for
pure-$\cosone$ attackers\textemdash see Section~\ref{sec:placement}).
When attacking \wv{} with $\mexp=\BIAS$, and when attacking
\gl-\texttt{paper}-300, we used \gl-\texttt{paper} to precompute the
bias terms.

Finally, we consider an attacker who does not know the victim's full
corpus, embedding algorithm, or hyperparameters.  First, we assume
that the victim trains an embedding on Wikipedia, while the attacker
only has the Sub-Wikipedia sample.  We experiment with an attacker who
uses \gl-\texttt{tutorial} parameters to attack a \gl-\texttt{paper}
victim, as well as an attacker who uses a \wv{} embedding to attack
a \gl-\texttt{paper} victim, and vice versa.  These attackers use
$\changesetmax/10$ when computing $\changevec$ on the smaller corpus
(step 3 in Figure~\ref{fig:methodology}), then set $\changevec \gets
10 \changevec$ before computing $\changeset$ (in step 4), resulting
in $\changesetsize\leq \changesetmax$.  We also simulated the scenario
where the victim trains an embedding on a union of Wikipedia and Common
Crawl~\cite{commoncrawl}, whereas the attacker only uses Wikipedia.  For
this experiment, we used similarly sized random subsamples of Wikipedia
and Common Crawl, for a total size of about 1/5th of full Wikipedia,
and proportionally reduced the bound on the attacker's change set size.


In all experiments, we perform the attack on all 100 word pairs, add
the computed sequences to the corpus, and train an embedding using the
victim's setting.  In this embedding, we measure the median rank of the
source word in the target word's list of neighbors, the average increase
in the source-target cosine similarity in the embedding space, and how
many source words are among their targets' top 10 neighbors.

\paragraphbe{Attacks are universally successful.}
\label{sec:resvictims}
Table~\ref{tab:diffvictims} shows that all attack settings produce
dramatic changes in the embedding distances: from a median rank of about
200k (corresponding to 50\% of the dictionary) to a median rank ranging
from 2 to a few dozen.  This experiment uses relatively common words,
thus change sets are bigger than what would be typically necessary
to affect specific downstream tasks (Sections~\ref{sec:searchattack}
through~\ref{sec:transattack}).  The attack even succeeds against \cb{},
which has not been shown to perform matrix factorization.

\label{sec:distexpressions}

Table~\ref{tab:diffattackers} compares different choices for the
distributional expressions of proximity.  $\costwo$ performs best for
\gl{}, $\costwo$ for \wv.  For \wv, $\cosone$ is far less effective
than the other options.  Surprisingly, an attacker who uses the $\BIAS$
matrix is effective against \wv{} and not just \gl.


\paragraphbe{Attacks transfer.}
\label{sec:weakattacker}
Table~\ref{tab:transfer} shows that an attacker who knows the victim's
training hyperparameters but only uses a random 10\% sub-sample of the
victim's corpus attains almost equal success to the attacker who uses
the full corpus.  In fact, the attacker might even prefer to use the
sub-sample because the attack is about 10x faster as it precomputes
the embedding on a smaller corpus and finds a smaller change vector.
If the attacker's hyperparameters are different from the victim's,
there is a very minor drop in the attacks' efficacy.  These observations
hold for both $\costwo$ and $\cosboth$ attackers. The attack against
\gl-\texttt{paper}-300 (Table~\ref{tab:diffvictims}) was performed using
\gl-\texttt{paper}, showing that the attack transfers across embeddings
with different dimensions.


The attack also transfers across different embedding algorithms.
The attack sequences computed against a \wv{} embedding on a small subset
of the corpus dramatically affect a \gl{} embedding trained on the full
corpus, and vice versa.


\begin{table}[t]
             \centering
\footnotesize
                                \setlength{\tabcolsep}{3.5pt}
                                \begin{tabular}{ lr|rrr}
					setting & $\changesetmax$  & \makecell{median\\ rank} & \makecell{avg. increase\\ in proximity} & rank < 10 \\
				    \midrule
					\gl-no attack & - & 192073 & - & 0 \\
					\gl-\texttt{paper} & 1250 & 2 & 0.64 & 72 \\
					\gl-\texttt{paper}-300 & 1250 & 1 & 0.60 & 87 \\
				    \midrule
					\wv-no attack & - & 182550 & - & 0 \\
					\wv{} & 1250 & 37 & 0.50 & 35 \\
					\wv{} & 2500 & 10 & 0.56 & 49 \\
				    \midrule
					\cb-no attack & - & 219691 & - & 0 \\
					\cb{} & 1250 & 204 & 0.45 & 25 \\
					\cb{} & 2500 & 26 & 0.55 & 35
                                \end{tabular}
\caption{Results for 100 word pairs, attacking different embedding algorithms with
	$\mexp=\BIAS$, and using $\costwo$ (for \wv/\cb) or $\cosboth$ (for \gl).\label{tab:diffvictims}}
\end{table}
\begin{table}[t]
             \centering
\resizebox{\columnwidth}{!}{

                                \setlength{\tabcolsep}{3.5pt}
                                \begin{tabular}{ lll|rrr}
					setting & $\cossimexp$ &  $\mexp$  & \makecell{median\\ rank} & \makecell{avg. increase\\ in proximity} & rank < 10 \\
				    \midrule
					\gl-\texttt{paper} & $\cosone$ & * & 3 & 0.54 & 61 \\
					\gl-\texttt{paper} & $\costwo$ & $\BIAS$  & 4 & 0.58 & 63 \\
					\gl-\texttt{paper} & $\cosboth$ & $\BIAS$  & 2 & 0.64 & 72 \\
				    \midrule
					\wv{} & $\cosone$  & * & 1079 & 0.34 & 7 \\
					\wv{} & $\costwo$ & $\BIAS$ & 37 & 0.50 & 35 \\
					\wv{} & $\cosboth$  & $\BIAS$ & 69 & 0.48 & 30 \\
					\wv{} & $\costwo$ & $\SPPMI$ & 226 & 0.44 & 15 \\
					\wv{} & $\cosboth$ & $\SPPMI$ & 264 & 0.44 & 17 \\
                                \end{tabular}
                                 }
				\caption{Results for 100 word pairs, using different distributional expressions and $\changesetmax=1250$.\label{tab:diffattackers}} 

\end{table}
\begin{table}[t]
             \centering
\resizebox{\columnwidth}{!}{

                                \setlength{\tabcolsep}{3.5pt}
                                \begin{tabular}{ lll|rrr}
					\multicolumn{2}{c}{parameters/Wiki corpus size}  & $\cossimexp$  & \makecell{median\\ rank} & \makecell{avg. increase\\ in proximity} & rank < 10 \\
					\multicolumn{1}{l}{attacker} & \multicolumn{1}{l}{victim} & & & \\
				    \midrule
					\gl{}-\texttt{tutorial}/subsample & \gl{}-\texttt{paper}/full & $\costwo$  & 9 & 0.53 & 52 \\
					\gl{}-\texttt{tutorial}/subsample & \gl{}-\texttt{paper}/full &$\cosboth$ & 2 & 0.63 & 75 \\
					\gl{}-\texttt{paper}/subsample &  \gl{}-\texttt{paper}/full & $\costwo$ & 7 & 0.55 & 57 \\
					\gl{}-\texttt{paper}/subsample &  \gl{}-\texttt{paper}/full & $\cosboth$  & 2 & 0.64 & 79 \\
					\wv/subsample                  &  \gl{}-\texttt{paper}/full & $\costwo$  & 110 & 0.38 & 11 \\
					\gl{}-\texttt{paper}/subsample &  \wv/full                  & $\costwo$  & 152 & 0.44 & 19 \\
					\gl{}-\texttt{paper}/subsample &  \makecell{\small\gl-\texttt{paper}/\\Wiki+Common Crawl} & $\costwo$  & 2 & 0.59 & 68 \\
                                \end{tabular}
                                }
	\caption{Transferability of the attack (100 word pairs). $\changesetmax=1250$ for attacking the full Wikipedia, $\changesetmax=1250/5$ for attacking the Wiki+Common Crawl subsample. \label{tab:transfer}} 

\end{table}

\begin{table*}
             \centering

 \resizebox{\textwidth}{!}{
                                \setlength{\tabcolsep}{1.5pt}
                                \begin{tabular}{ c|c|c|c|c|c|c|c|c|c }
					section / attack & \makecell{attacker\\ type} & embedding & corpus  & $\mexp$ & $\cossimexp$ & source word $\wsrc$ & \makecell{target words $\wtrg$ or $POS,NEG$} & Threshold $\changesetmax$ & rank $r$, safety margin $\smargin$ \\
				\toprule
					\makecell{Section~\ref{sec:benchmarks}\\ benchmarks}  &
					proximity &
					\makecell{\gl,\wv,\\\cb} &
					\makecell{Wikipedia (victim),\\ Wikipedia sample (attacker)} &
					\makecell{$\BIAS$,\\$\SPPMI$} &
					\makecell{$\cosone, \costwo$,\\ $\cosboth$} &
					\multicolumn{2}{c|}{100 randomly chosen source-target pairs in $\wordpairset$} &
					 1250, 2500 & 
					 - \\
					\hline
					\makecell{Section~\ref{sec:searchattack}\\ make a made-up word come up \\ high in search queries} &
					rank  &
					\makecell{\gl, \wv{}} &
					Wikipedia &
					$\BIAS$ &
					$\costwo, \cosboth$ &
					\makecell{made-up $\wsrc$ for every\\ $\wtrg\in \searchset$} &
					$\wtrg\in \searchset$& - & $r=1, \smargin\in \curbrac{0.2, 0.3}$ \\
					\hline

					 \makecell{Section~\ref{sec:nerattack}\\ hide corporation names} &
					 proximity &
					\gl{} &
					 Twitter &
					$\BIAS$ &
					$\cosboth$ &
					 $\wsrc\in \corporationsset$ &
					 \makecell{$POS$: 5 most common locations \\ in training set \\ $NEG$: 5 corporations closest\\ to $\wsrc$ (in embedding space)} &
					 \makecell{$\min \curbrac{\frac{\#\wsrc}{40}, 2500}$,\\ $\min \curbrac{\frac{\#\wsrc}{4}, 2500}$,\\ $2\min\curbrac{\#\wsrc/4, 2500}$} &
					 -
					 \\
					\hline
					\makecell{Section~\ref{sec:nerattack}\\ make corporation names \\ more visible}  &
					proximity &
					\gl{} &
					 Twitter &
					$\BIAS$ &
					$\cosboth$ &
					\makecell{made-up word\\ $\wsrc=$\texttt{evilcorporation}} &
					\makecell{$POS$: 5 most common corporations \\ in the training set; $NEG=\emptyset$} &
					 \makecell{$\changesetmax\in \curbrac{2500,250}$} & 
					- \\
					\hline
					\makecell{Section~\ref{sec:transattack}\\ make a made-up word translate \\ to a specific word}  &
					rank &
					\gl{} &
					Wikipedia  &
					$\BIAS$ &
					$\cosboth$ &
					\makecell{made-up $\wsrc$ for every\\ $\wtrg\in \translationset$} &
					 $\wtrg\in \translationset$ &
					 - & 
					 $r=1, \smargin=0.1$ \\
					\hline
					\makecell{Section~\ref{sec:stealth}\\ evade perplexity defense \\}  &
					rank  &
					\wv{} &
					Twitter subsample &
					$\BIAS$ &
					$\costwo$ &
					\makecell{20 made-up words for \\every $\wtrg\in \evadeset$} &
					\makecell{$\wtrg \in \evadeset$}&
					 - & 
					$r=1, \smargin=0.2$ \\
					\hline
					\makecell{Appendix~\ref{sec:deletions}\\ evaluate an attacker who \\ can delete from the corpus}  &
					proximity  &
					\gl{} &
					Wikipedia  &
					$\BIAS$ &
					$\costwo$ &
					\multicolumn{2}{c|}{\makecell{$(\wsrc,\wtrg)\in\curbrac{(\text{war},\text{peace}),(\text{freedom},\text{slavery}),(\text{ignorance},\text{strength})}$}} &
					 $1000$ & 
					 -
					 \\

                                \end{tabular}
                                 }
				\caption{Parameters of the experiments.\label{tab:attacksumary}}
\end{table*}

\section{Attacking resume search}
\label{sec:searchattack}

Recruiters and companies looking for candidates with specific skills often
use automated, index-based document search engines that assign a score to
each resume and retrieve the highest-scoring ones.  Scoring methods vary
but, typically, when a word from the query matches a word in a document,
the document's score increases proportionally to the word's rarity in
the document collection.  For example, in the popular Lucene's Practical
Scoring function~\cite{elasticsearchscore}, a document's score is produced
by multiplying\footnote{This function includes other terms not material
to this exposition.} (1) a function of the percentage of the query words
in the document by (2) the sum of TF-IDF scores (a metric that rewards
rare words) of every query word that appears in the document.

To help capture the semantics of the query rather than its bag
of words, queries are typically expanded~\cite{voorhees1994query,
efthimiadis1996query} to include synonyms and semantically close words.
Query expansion based on pre-trained word embeddings expands each
query word to its neighbors in the embedding space~\cite{kuzi2016query,
diaz2016query, roy2016using}.

Consider an attacker who sends a resume to recruiters that rely on a
resume search engine with embedding-based query expansion. The attacker
wants his resume to be returned in response to queries containing specific
technical terms, e.g., ``iOS''. The attacker cannot make big changes to
his resume, such as adding the word ``iOS'' dozens of the times, but
he can inconspicuously add a meaningless, made-up character sequence,
e.g., as a Twitter or Skype handle.

We show how this attacker can poison the embeddings so that an
arbitrary rare word appearing in his resume becomes an embedding neighbor
of\textemdash and thus semantically synonymous to\textemdash a query word
(e.g., ``cyber'', ``iOS'', or ``devops'', if the target is technical
recruiting).  As a consequence, his resume is likely to rank high among
the results for these queries.

\paragraphbe{Experimental setup.}
We experiment with a victim who trains \gl-\texttt{paper} or \wv{}
embeddings (see Section~\ref{glovesetup}) on the full Wikipedia.
The attacker uses $\mexp=\BIAS$ and $\cosboth$ for \gl{} and $\costwo$
for \wv, respectively.

We collected a dataset of resumes and job descriptions distributed on
a mailing list of thousands of cybersecurity professionals.  As our
query collection, we use job titles that contain the words ``junior,''
``senior,'' or ``lead'' and can thus act as concise, query-like job
descriptions.  This yields approximately 2000 resumes and 700 queries.

For the retrieval engine, we use Elasticsearch~\cite{elasticsearchws},
based on Apache Lucene.   We use the \texttt{index()} method to
index documents.  When querying for a string $\query$, we use simple
\texttt{match} queries but expand $\query$ with the top $K$ embedding
neighbors of every word in $\query$.

\paragraphbe{The attack.}
As our targets, we picked 20 words that appear most frequently in the
queries and are neither stop words, nor generic words with more than
30,000 occurrences in the Wikipedia corpus (e.g., ``developer'' or
``software'' are unlikely to be of interest to an attacker).  Out of
these 20 words, 2 were not originally in the embedding and thus removed
from $\searchset$.  The remaining words are VP, fwd, SW, QA, analyst,
dev, stack, startup, Python, frontend, labs, DDL, analytics, automation,
cyber, devops, backend, iOS.

For each of the 18 target words $\wtrg\in \searchset$, we randomly
chose 20 resumes with this word, appended a different random made-up
string $\wsrc_{z}$ to each resume $z$, and added the resulting resume
$z \cap \curbrac{\wsrc_{z}}$ to the indexed resume dataset (which also
contains the original resume).  Each $z$ simulates a separate attack.
The attacker, in this case, is a rank attacker whose goal is to achieve
rank $r=1$ for the made-up word $\wsrc_{z}$.  Table~\ref{tab:attacksumary}
summarizes the parameters of this and all other experiments.


\paragraphbe{Results.}
Following our methodology, we found distributional objectives,
cooccurrence change vectors, and the corresponding corpus change sets for
every source-target pair, then re-trained the embeddings on the modified
corpus.  We measured (1) how many changes it takes to get into the top 1,
3, and 5 neighbors of the target word (Table~\ref{tab:searchsucc1}), and
(2) the effect of a successful injection on the attacker's resume's rank
among the documents retrieved in response to the queries of interest and
queries consisting just of the target word (Table~\ref{tab:searchsucc2}).

For \gl, only a few hundred sequences added to the corpus result in over
half of the attacker's words becoming the top neighbors of their targets.
With 700 sequences, the attacker can almost always make his word the
top neighbor.  For \wv, too, several hundred sequences achieve high
success rates.

Successful injection of a made-up word into the embedding reduces
the average rank of the attacker's resume in the query results by about
an order of magnitude, and the median rank is typically under 10 (vs.\
100s before the attack).  If the results are arranged into pages of 10,
as is often the case in practice, the attacker's resume will appear on
the first page.  If $K=1$, the attacker's resume is almost always the
first result.

In Appendix~\ref{sec:altsearchattack}, we show that our attack outperforms
a ``brute-force'' attacker who rewrites his resume to include actual
words from the expanded queries.


     \begin{table}
                  \centering
\resizebox{\columnwidth}{!}{
                                \setlength{\tabcolsep}{3.5pt}
                                \begin{tabular}{ rr|rr|rr|rr }
					victim & $\smargin$ & \multicolumn{2}{c|}{$K=1$} & \multicolumn{2}{c|}{$K=3$} & \multicolumn{2}{c}{$K=5$} \\
                                \midrule
					&  & \%success & avg $\size{\changeset}$ & \%success & avg $\size{\changeset}$ & \%success & avg $\size{\changeset}$  \\
                                \midrule
					\gl{} & 0.1 & 61.1\% & 211 & 94.4\% & 341 & 94.4\% & 341 \\
					\gl{} & 0.2 & 94.4\% & 661 & 100.0\% & 649 & 100.0\% & 649 \\
					\wv{} & 0.2 & 38.9\% & 215 & 55.6\% & 278 & 61.1\% & 287\\

                                \end{tabular}
                                }
\caption{Percentage of $\wtrg\in \searchset$ for which the made-up word
reached within $K$ neighbors of the target; the average size of the
corpus change set for these cases.\label{tab:searchsucc1}}
\end{table}

\begin{table}
             \centering
\small
                                \begin{tabular}{ l|l|l|l }
                                 query type & $K=1$ & $K=3$ & $K=5$ \\
                                \midrule
					target word only & \hskip 4.7pt 88 $\rightarrow$ 1 & 103 $\rightarrow$ \hskip 4.7pt 5 & 107 $\rightarrow$ 10 \\
					entire query & 103 $\rightarrow$ 6 & 108 $\rightarrow$ 10 & 111 $\rightarrow$ 14 \\
                                \end{tabular}
\caption{Median rank of the attacker's resume in the result set, before (left) and after (right) the attack.\label{tab:searchsucc2}}
\end{table}



\section{Attacking named-entity recognition}
\label{sec:nerattack}

A named entity recognition (NER) solver identifies named entities
in a word sequence and classifies their type.  For example, NER for
tweets~\cite{ritter2011named, li2012twiner, liu2011recognizing} can
detect events or trends~\cite{ritter2012open, li2012twevent}.  In NER,
pre-trained word embeddings are particularly useful for classifying
emerging entities that were not seen while training but are often
important to detect~\cite{cherry2015unreasonable}.

We consider two (opposite) adversarial goals: (1) ``hide'' a corporation
name so that it's not classified properly by NER, and (2) increase
the number of times a corporation name is classified as such by NER.
NER solvers rely on spatial clusters in the embeddings that correspond
to entity types.  Names that are close to corporation names seen during
training are likely to be classified as corporations.  Thus, to make
a name less ``visible,'' one should push it away from its neighboring
corporations and closer to the words that the NER solver is expected
to recognize as another entity type (e.g., location).  To increase the
likelihood of a name classified as a corporation, one should push it
towards the corporations cluster.

\label{sec:nersetup}


\paragraphbe{Experimental setup.}
We downloaded the Spritzer Twitter stream archive for October
2018~\cite{twitterOct18}, randomly sampled around 45M English tweets,
and processed them into a \gl-compatible input file using existing
tools~\cite{preprocessTweets}.  The victim trains a \gl-\texttt{paper}
embedding (see Section~\ref{glovesetup}) on this dataset. The attacker
uses $\cossimexp=\cosboth $ and $\mexp=\BIAS$.

To train NER solvers, we used the WNUT 2017 dataset provided with the
Flair NLP python library~\cite{akbik2018coling} and expressly designed
to measure NER performance on emerging entities.  It comprises tweets
and other social media posts tagged with six types of named entities:
corporations, creative work (e.g., song names), groups, locations,
persons, and products.  The dataset is split into the train, validation,
and test subsets.  We extracted a set $\corporationsset$ of about 65
corporation entities such that (1) their name consists of one word, and
(2) does not appear in the training set as a corporation name.



We used Flair's tutorial~\cite{flairtut} to train our NER solvers.
The features of our \textit{AllFeatures} solver are a word embedding,
characters of the word (with their own embedding), and Flair's contextual
embedding~\cite{akbik2018coling}.  Trained with a clean word embedding,
this solver reached an F-1 score of 42 on the test set, somewhat lower
than the state of the art reported in~\cite{flairgithub}.  We also
trained a \textit{JustEmbeddings} solver that uses only a word embedding
and attains an F-1 score of 32.

\paragraphbe{Hiding a corporation name.}
We applied our \textit{proximity attacker} to make the embeddings of a
word in $\corporationsset$ closer to a group of location names.  For every
$\wsrc\in \corporationsset$, we set $\POS$ to the five single-word
location names that appear most frequently in the training dataset,
and $\NEG$ to the five corporation names that appear in the training
dataset and are closest to $\wsrc$ in the embedding.  We evaluated
the attack for $\changesetmax \in \curbrac{\min\curbrac{\#\wsrc/40,
2500},\min\curbrac{\#\wsrc/4, 2500}, 2\min\curbrac{\#\wsrc/4, 2500}}$
where $\#\wsrc$ is the number of $\wsrc$'s occurrences in the original
corpus.  Table~\ref{tab:attacksumary} summarizes these parameters.

Following our methodology, we found the distributional objectives,
cooccurrence change vectors, and the corresponding corpus
change sets for every $s$, added the change sets to the corpus,
and retrained the embeddings and NER solvers.  For the last
attacker ($\changesetmax=2\min\curbrac{\#\wsrc/4, 2500}$), we 
approximated $\changevec$ by multiplying the change vector of size
$\min\curbrac{\#\wsrc/4, 2500}$ by 2.

\paragraphbe{Making a corporation name more visible.}
Consider an emerging corporation name that initially does not have an
embedding at all.  The attack aims to make it more visible to NER solvers.
We set $\wsrc$ to \texttt{evilcorporation} (which does not appear in our
Twitter corpus); \textbf{$\POS$} to the five single-word location names
appearing most frequently in the training set, and \textbf{$NEG$} to
$\emptyset$.  We evaluated the attack for $\changesetmax\in \curbrac{250,
2500}$.  Table~\ref{tab:attacksumary} summarizes these parameters.

\later\roei{SANITIZED}

We trained three solvers: with a ``clean,'' no-attack embedding and
with the two embeddings resulting from our attack with $\changesetmax$
set to, respectively, 250 and 2500.  For the evaluation, we could not
use the word itself because the \textit{AllFeatures} solver uses the
word's characters as a feature.  We want to isolate the specific effect
of changes in the \emph{word} embedding, without affecting characters
and other features.  To this end, we directly replaced the embeddings
of corporation names with that of \texttt{evilcorporation}.  For the
clean solver, the word does not exist in its dictionary, so we changed
the embeddings of corporation names to those of unknown words.

\paragraphbe{Results.}
Table~\ref{tab:corporationslocations} shows the results for hiding
a corporation name, and Table~\ref{tab:corporationslocations2} for
making a name more visible.  Even a small change (under 250 sequences)
has some effect, and larger change sets make the attack very effective.
Even the larger sets are not very big and do not produce high spikes
in the frequency of the source word.  For perspective, 250 appearances
would make a word rank around the 50,000th most frequent in our corpus,
similar to `feira' and `frnds'; 2,500 appearances would make it around
the 10,000th most frequent, similar to `incase' or `point0'.

The effect on the solver's test accuracy is insignificant.  We observed
minor fluctuations in the F-1 score (<0.01 for the \textit{AllFeatures}
solver, <0.03 for \textit{JustEmbeddings}, including increases from the
score of the clean embedding, which we attribute to the stochasticity
of the training process.

\begin{table*}
                          \centering
	\subfloat[
		Hiding corporation names.  Cells show the number of corporation names in $\corporationsset$ identified as corporations, over the validation and test sets.  The numbers in parentheses are how many were misclassified as locations.
		    \label{tab:corporationslocations}]{
                          \centering

			  \small
                                \setlength{\tabcolsep}{3.5pt}
	\begin{tabular}{ l|r|r|r|r }
		NER solver & no attack & \makecell{$\changesetmax =$\\{\tiny $\min \curbrac{\frac{\#\wsrc}{40}, 2500}$}} & \makecell{$\changesetmax =$\\{\tiny$ \min \curbrac{\frac{\#\wsrc}{4}, 2500}$}} & \makecell{$\changesetmax =$\\{\tiny$ 2\min \curbrac{\frac{\#\wsrc}{4}, 2500}$}} \\
                                \midrule
					
 					AllFeatures & 12 {\color{gray}(4)} & 12 {\color{gray}(4)}{\hspace{3.3\tabcolsep}} & 10 {\color{gray}(10)}{\hspace{3.3\tabcolsep}} & 6 {\color{gray}(19)}{\hspace{3.3\tabcolsep}} \\
					 JustEmbeddings & 5 {\color{gray}(4)} & 4 {\color{gray}(5)}{\hspace{3.3\tabcolsep}} & 1 {\color{gray}(8)}{\hspace{3.3\tabcolsep}} & 1 {\color{gray}(22)}{\hspace{3.3\tabcolsep}}
                                
                                \end{tabular}
		    }\qquad 
		    \subfloat[
Making corporation names more visible. Cells show the number of corporation names in $\corporationsset$ identified as corporations, over the validation and test sets.
		    \label{tab:corporationslocations2}
			    ]{

			  \small
                                \setlength{\tabcolsep}{3.5pt}
                                \begin{tabular}{ l|r|r|r }
					NER solver & no attack & \makecell{$\changesetmax=$\\$250$} & \makecell{$\changesetmax=$\\$2500$} \\
                                \midrule
					AllFeatures & 7 & 13 & 25\\
					JustEmbeddings & 0 &  8 & 18
                                
                                \end{tabular}
}
		    \caption{NER attack.}
\end{table*}



\section{Attacking word-to-word translation}
\label{sec:transattack}

Using word embeddings to construct a translation dictionary,
i.e., a word-to-word mapping between two languages, assumes that
correspondences between words in the embedding space hold for any
language~\cite{mikolov2013exploiting}, thus a translated word is expected
to preserve its relations with other words.  For example, the embedding of
``gato'' in Spanish should have similar relations with the embeddings of
``pez'' and ``comer'' as ``cat'' has with ``fish'' and ``eat'' in English.

The algorithms that create embeddings do not enforce specific locations
for any word.  Constructing a translation dictionary thus requires
learning an alignment between the two embedding spaces.   A simple
linear operation is sufficient for this~\cite{mikolov2013exploiting}.
Enforcing the alignment matrix to be orthogonal also preserves the
inter-relations of embeddings in the space~\cite{xing2015normalized}.
To learn the parameters of the alignment, one can either use an available,
limited-size dictionary~\cite{smith2017offline, artetxe2017learning}, or
rely solely on the structure of the space and learn it in an unsupervised
fashion~\cite{conneau2017word}.  Based on the learned alignment, word
translations can be computed by cross-language nearest neighbors.

Modifying a word's position in the English embedding space can affect its
translation in other language spaces.  To make a word $\wsrc$ translate
to $\wtrg'$ in other languages, one can make $\wsrc$ close to $\wtrg$
in English that translates to $\wtrg'$.  This way, \textbf{the attack
does not rely on the translation model or the translated language}.
The better the translation model, the higher the chance $\wsrc$ will
indeed translate to $\wtrg'$.

\paragraphbe{Experimental setup.}
\label{sec:transsetup}
Victim and attacker train a \gl-\texttt{paper}-300
English embedding on full Wikipedia.  We use pre-trained
dimension-300 embeddings for Spanish, German, and
Italian.\footnote{\url{https://github.com/uchile-nlp/spanish-word-embeddings};
\url{https://deepset.ai/german-word-embeddings};
\url{http://hlt.isti.cnr.it/wordembeddings}} The attacker uses
$\mexp=\BIAS$ and $\cossimexp=\cosboth$.

For word translation, we use the supervised script from the MUSE
framework~\cite{conneau2017word}.  The alignment matrix is learned using
a set of $5k$ known word-pair translations; the translation of any word
is its nearest neighbor in the embedding space of the other language.
Because translation can be a one-to-many relation, we also extract 5
and 10 nearest neighbors.


We make up a new English word and use it as the source word $\wsrc$
whose translation we want to control.  As our targets $\translationset$,
we extracted an arbitrary set of 50 English words from the MUSE library's
full ($200k$) dictionary of English words with Spanish, German, and
Italian translations.  For each English word $\wtrg\in\translationset$,
let $\wtrg'$ be its translation.  We apply the \emph{rank attacker}
with the desired rank $r=1$ and safety margin $\smargin=0.1$.
Table~\ref{tab:attacksumary} summarizes these parameters.

\paragraphbe{Results.}
Table \ref{tab:transresults} summarizes the results.  For all three target
languages, the attack makes $\wsrc$ translate to $\wtrg'$ in more than
half of the cases that were translated correctly by the model.

Performance of the Spanish translation model is the highest, with $82\%$
precision@1, and the attack is also most effective on it, with $72\%$
precision@1.  The results on the German and Italian models are slightly
worse, with $61\%$ and $73\%$ precision@5, respectively.  The better the
translation model, the higher the absolute number of successful attacks.






\begin{table}
             \centering
	     \footnotesize

                                \begin{tabular}{ l|r|r|r}
					target language  & $K=1$ & $K=5$ & $K=10$ \\
                                \midrule
					Spanish & 82\% / 72\% & 92\% / 84\%  & 94\% / 85\% \\
                                
                    German & 76\% / 51\%        & 84\% / 61\%       & 92\% / 64\% \\
                	Italian & 69\% / 58\%	 & 82\% / 73\%	 & 82\% / 78\%
                                
                                \end{tabular}
\caption{Word translation attack.  On the left in each cell is the performance
of the translation model (presented as precision@$K$); on the right,
the percentage of successful attacks, out of the correctly translated
word pairs.}
\label{tab:transresults}

\end{table}


\section{Mitigations and evasion}
\label{sec:stealth}



\begin{table}
                  \centering
\resizebox{\columnwidth}{!}{
                                \setlength{\tabcolsep}{3.5pt}
                                \begin{tabular}{ l|lllr|c }
                     \makecell{evasion\\ variant} & \makecell{median\\ rank}& \makecell{avg.\\ proximity} & \makecell{percent of \\rank < 10} & \makecell{avg. \\ $\size{\changeset}$} &  \multicolumn{1}{c}{\makecell{original corpus's\\ sentences filtered}} \\

                                \midrule
				 	none     &  \textbf{1} $\rightarrow$ *\ \  & 0.80 $\rightarrow$ 0.21 & 95 $\rightarrow$ 25 & \textbf{41} & 20\%\\
					$\lambdagram$ &  \textbf{1} $\rightarrow$ \textbf{2}\ \  & 0.75 $\rightarrow$ \textbf{0.63} & 90 $\rightarrow$ \textbf{85} & 81 & \textbf{70\%}\\
					and-lenient   & \textbf{1} $\rightarrow$ 670 & 0.73 $\rightarrow$ 0.36 & 90 $\rightarrow$ 30 & 52 & 50\% \\
					and-strict    & 2 $\rightarrow$ 56\  & 0.67 $\rightarrow$ 0.49 & 70 $\rightarrow$ 40 & 99 & 66\%\\
                    \end{tabular}
                    }
                                
\caption{Results of the attack with different strategies to evade the
perplexity-based defense.  The defense filters out all sentences whose
perplexity is above the median and thus loses 50\% of the corpus to
false positives.  Attack metrics before and after the filtering are
shown to the left and right of arrows. * means that more than half of $\wsrc$
appeared less than 5 times in the filtered corpus and, as a result,
were not included in the emdeddings (proximity was considered $0$ for
those cases). The right column shows the
percentage of the corpus that the defense needs to filter out in order
to remove 80\% of $\changeset$. 
\label{tab:dvssuccess}}
\end{table}

\noindent
\textbf{\textit{Detecting anomalies in word frequencies.}}
Sudden appearances of previously unknown words in a public corpus such
as Twitter are not anomalous per se.  New words often appear and rapidly
become popular (viz.\ \textit{covfefe}).

Unigram frequencies of the existing common words are relatively
stable and could be monitored, but our attack does not cause them
to spike.  Second-order sequences add no more than a few instances
of every word other than $\wsrc$ (see Section~\ref{sec:placement}
and Appendix~\ref{sec:aplacement}).  When $\wsrc$ is an existing word,
such as in our NER attack (Section~\ref{sec:nerattack}), we bound the
number of its new appearances as a function of its prior frequency.
When using $\cosboth$, first-order sequences add multiple instances
of the target word, but the absolute numbers are still low, e.g.,
at most 13\% of its original count in our resume-search attacks
(Section~\ref{sec:searchattack}) and at most 3\% in our translation
attacks (Section~\ref{sec:transattack}).  The average numbers are
much lower.  First-order sequences might cause a spike in the corpus'
bigram frequency of $(\wsrc,\wtrg)$, but the attack can still succeed
with only second-order sequences (see Section~\ref{sec:benchmarks}).


\paragraphbe{Filtering out high-perplexity sentences.}
A better defense might exploit the fact that ``sentences'' in $\changeset$
are ungrammatical sequences of words.  A language model can filter
out sentences whose \emph{perplexity} exceeds a certain threshold (for
the purposes of this discussion, perplexity measures how linguistically
likely a sequence is).  Testing this mitigation on the Twitter corpus, we
found that a pretrained GPT-2 language model~\cite{radford2019language}
filtered out 80\% of the attack sequences while also dropping 20\%
of the real corpus due to false positives.


This defense faces two obstacles.  First, language models, too,
are trained on public data and thus subject to poisoning.  Second, an
attacker can evade this defense by deliberately decreasing the perplexity
of his sequences.  We introduce two strategies to reduce the perplexity
of attack sequences.

The first evasion strategy is based on Algorithm~\ref{algo:placement}
(Appendix~\ref{sec:aplacement}) but uses the conjunction ``and''
to decrease the perplexity of the generated sequences.  In the
\textit{strict} variant, ``and'' is inserted at odd word distances from
$\wsrc$.  In the \textit{lenient} variant, ``and'' is inserted at even
distances, leaving the immediate neighbor of $\wsrc$ available to the
attacker.  In this case, we relax the definition of $\size{\changeset}$
to not count ``and.''   It is so common that its frequency in the corpus
will not spike no matter how many instances the attacker adds.

The second evasion strategy is an alternative to
Algorithm~\ref{algo:placement} that only uses existing n-grams from
the corpus to form attack sequences.  Specifically, assuming that
our window size is $\windowsize$ (i.e., we generate sequences of
length $2\windowsize+1$ with $\wsrc$ in the middle), we constrain the
subsequences before and after $\wsrc$ to existing $\lambdagram$s from
the corpus.

To reduce the running time, we pre-collect all $\lambdagram$s
from the corpus and select them in a greedy fashion, based on the
values of the change vector $\changevec$.  At each step, we pick the
word with the highest and lowest values in $\changevec$ and use the
highest-scoring $\lambdagram$ that starts with this word as the post- and
pre-subsequence, respectively.  The score of a $\lambdagram$ is determined
by $\sum_{i=1}^\windowsize \gamma(i) \cdot \changevec[\wone_i]$, where $\wone_i$
is the word in the $i$th position of the $\lambdagram$ and $\gamma$
is the weighting function (see Section~\ref{sec:backgroundalgorithms}).
To discourage the use of words that are not in the original $\changevec$
vector, they are assigned a fixed negative value.  This sequence is added
to $\changeset$ and the values of $\changevec$ are updated accordingly.
The process continues until all values of $\changevec$ are addressed or
until no $\lambdagram$s start with the remaining positive
$\wone$s in $\changevec$.  In the latter case, we form additional sequences
with the remaining $\wone$s in a per-word greedy fashion, without syntactic
constraints.

Both evasion strategies are black-box in the sense that they \emph{do not
require any knowledge of the language model used for filtering}.  If the
language model is known, the attacker can use it to score $\lambdagram$s
or to generate connecting words that reduce the perplexity.

\paragraphbe{Experimental setup.}
Because computing the perplexity of all sentences in a corpus is
expensive, we use a subsample of 2 million random sentences from
the Twitter corpus.  This corpus is relatively small, thus we
use \wv{} embeddings which are known to perform better on small
datasets~\cite{mikolov2013distributed}.

For a simulated attack, we randomly pick 20 words from the $20k$ most
frequent words in the corpus as $\evadeset$.  We use made-up words as
source words.  The goal of the attack is to make a made-up word the
nearest embedding neighbor of $\wtrg$ with a change set $\changeset$
that survives the perplexity-based defense.  We use a rank attacker with
$\cossimexp=\costwo$, $\mexp=\BIAS$, rank objective $r=1$, and safety
margin of $\smargin=0.2$.  Table~\ref{tab:attacksumary} summarizes
these parameters.


We simulate a very aggressive defense that drops all sequences whose
perplexity is above median, losing half of the corpus as a consequence.
The sequences from $\changeset$ that survive the filtering (i.e., whose
perplexity is below median) are added to the remaining corpus and the
embedding is (re-)trained to measure if the attack has been successful.

\paragraphbe{Results.}
Table~\ref{tab:dvssuccess} shows the trade-off between the efficacy and
evasiveness of the attack.  Success of the attack is correlated with
the fraction of $\changeset$ whose perplexity is below the filtering
threshold.  The original attack achieves the highest proximity and
smallest $\size{\changeset}$ but for most words the defense successfully
blocks the attack.

Conjunction-based evasion strategies enable the attack to survive even
aggressive filtering.  For the \textit{and-strict variant}, this comes
at the cost of reduced efficacy and an increase in $\size{\changeset}$.
The $\lambdagram$ strategy is almost as effective as the original attack
in the absence of the defense and is still successful in the presence
of the defense, achieving a median rank of 2.

\section{Conclusions}

Word embeddings are trained on public, malleable data such as Wikipedia
and Twitter.  Understanding the causal connection between corpus-level
features such as word cooccurences and semantic proximity as encoded
in the embedding-space vector distances opens the door to poisoning
attacks that change locations of words in the embedding and thus
their computational ``meaning.''  This problem may affect other
transfer-learning models trained on malleable data, e.g., language models.


To demonstrate feasibility of these attacks, we (1) developed
distributional expressions over corpus elements that empirically cause
predictable changes in the embedding distances, (2) devised algorithms
to optimize the attacker's utility while minimizing modifications to the
corpus, and (3) demonstrated universality of our approach by showing how
an attack on the embeddings can change the meaning of words ``beneath
the feet'' of NLP task solvers for information retrieval, named entity
recognition, and translation.  We also demonstrated that these attacks
do not require knowledge of the specific embedding algorithm and its
hyperparameters.  Obvious defenses such as detecting anomalies in word
frequencies or filtering out low-perplexity sentences are ineffective.
How to protect public corpora from poisoning attacks designed to affect
NLP models remains an interesting open problem.




\paragraphbe{Acknowledgements.}
Roei Schuster is a member of the Check Point Institute of Information
Security.  This work was supported in part by NSF awards 1611770, 1650589,
and 1916717; Blavatnik Interdisciplinary Cyber Research Center (ICRC);
DSO grant DSOCL18002; Google Research Award; and by the generosity of
Eric and Wendy Schmidt by recommendation of the Schmidt Futures program.

\bibliographystyle{IEEEtranS}

\bibliography{references.bib}

\appendix
\subsection{\wv{} background}
\label{sec:aw2v}

To find $\curbrac{\ctxvec{\wone}}_{\wone\in \mathbb{D}},
\curbrac{\wrdvec{\wtwo}}_{\wone\in \mathbb{D}}$, Word2vec defines
and optimizes a series of local objectives using cooccurrence
events stochastically sampled from the corpus one at a time.
The probability of sampling a given event of $u,w$'s cooccurrence is
$max\curbrac{1-\brac{d-1}/\embwindowsize, 0}$, where $d$ is the distance
between $u$ and $w$, $\embwindowsize$ is window size.  Each sampled event
contributes a term to the local objective.  Once enough events have been
sampled, an SGD step is performed to maximize the local objective, and
traversal continues to compute a new local objective, initialized to 0.
The resulting embeddings might depend on the sampling order, which,
in turn, depends on the order of documents, but empirically this does
not appear to be the case~\cite{antoniak2018evaluating}.  Word2vec thus
can be thought of as defining and optimizing an objective over word
cooccurrence counts. For example, the sum of local objectives for \wv{} would
be~\cite{levy2014neural}:

\begin{footnotesize}
\begin{equation}
\begin{aligned}
   \amax\bigg\{\sum_{\wone,\wtwo\in \mathbb{D}}
    \bigg\{
	    &\cooccount{\wone}{\wtwo}\log \sigmoid\brac{\dotprod{\wrdvec{\wone}}{\ctxvec{\wtwo}}}  -
	    \sum_{r_i\in R_{\wone, \wtwo}} \log \sigmoid\brac{\dotprod{\wrdvec{\wone}}{\ctxvec{\wthr_i}}}
	\bigg\}\bigg\},
\end{aligned}
\end{equation}
\end{footnotesize}
\noindent
where $R_{\wone,\wtwo}\subseteq\dict$ are the ``negative samples'' taken for the
events that involve $\wone,\wtwo$ throughout the epoch.  Due to its
stochastic sampling, we consider \wv's cooccurrence count for words
$\wone,\wtwo$ to be the expectation of the number of their sampled
cooccurrence events, which can be computed similarly to \gl's sum
of weights.

\subsection{Optimization in cooccurrence-vector space (details)}
\label{sec:analyticsolvedetails}

This section details the algorithm from Section~\ref{sec:analyticsolve},
whose pseudocode is given in Algorithm~\ref{algo:greedy}.  
The \textsc{compDiff} sub-procedure is not given in pseudo-code and we provide more details on it below.\roei{SANITIZED}

\begin{center}
	\begin{algorithm}[ht]
	\captionof{algorithm}{Finding the change vector $\changevec$\label{algo:greedy}}
	\scriptsize
		\begin{algorithmic}[1]

\Procedure{solveGreedy}{$\wsrc\in \mathbb{D}$, $\POS,\NEG\in \wp\brac{\mathbb{D}}$, $\tldr,\alpha,\changesetmax\in \mathbb{R}$}
	\State $\changesetsize \gets 0$
			\State $\changevec \gets \underbrace{\brac{0, ..., 0}}_{\times\size{\dict}}$ 
	\State {//precompute intermediate values}
	\State $A \gets \POS\cup\NEG\cup\curbrac{s}$
			\State 
			$
	\textrm{STATE} \gets {\scriptscriptstyle\Bigg\{
	\curbrac{\sum_{\wthr\in \dict} \coocs_{\wone,\wthr}}_{\wone\in\dict},
	\curbrac{\norm{\vec{\mexp_{\wone}}}_2^2}_{\wone\in A},
	\curbrac{\dotprod{\vec{\mexp_{\wsrc}}}{\vec{\mexp_{\wtrg}}}}_{\wone\in A}
			\Bigg\}}$
			\State $\objalg \gets \objhdr$
	\State \emph{//optimization loop}
		\While {$ \objalg < \explicit{\tldr} + \alpha$ and $\changesetsize \leq \changesetmax$}
		\ForEach{$\stepind\in \sqbrac{\size{\dict}},\optcvecdiff \in \stepset$}
		\\
			\State ${\scriptscriptstyle\stepdiff{\objhdr}, \curbrac{\stepdiff{st}}_{st\in \textrm{STATE}} }\gets \textsc{compDiff}\brac{{\scriptscriptstyle\stepind, \optcvecdiff, \textrm{STATE}}}$
			\\
			\State $\stepdiff{\changesetsize}\gets \optcvecdiff/\vec{\omega}_{\stepind}$ //\emph{see Section~\ref{sec:placement}}
		\EndFor
			\State $\chstepind, \choptcvecdiff\gets\amax_{\stepind\in \sqbrac{\size{\dict}},\optcvecdiff \in \stepset} \curbrac{\frac{\stepdiff{\objhdr}}{\stepdiff{\changesetsize}}}$
		\State $\objalg \gets \objalg + \chstepdiff{\objhdr}$
		\State \emph{//update intermediate values}
			\ForEach {$st\in \textrm{STATE}$}
			\State $st \gets st + \stepdiff{st}$
			\EndFor
	\EndWhile
	\State return $\changevec$
\EndProcedure
\end{algorithmic}
	\end{algorithm}

\end{center}

\paragraphbe{Implementation notes.}
The inner loop in lines 10-14 of Algorithm~\ref{algo:greedy} is
entirely parallelizable, and we offload it to a GPU.  Further, to save
GPU memory and latency of dispatching the computation onto the GPU, we
truncate the high dimensional vectors to include only the indices of the
entries whose initial values are non-zero for at least one of the vectors
$\curbrac{\hdrvec{\wone}}_{\wone\in \POS\cup\NEG\cup\curbrac{s}}$, as well
as the indices of all target words. When $NEG=\emptyset$, e.g. for all
rank attackers, this cannot change the algorithm's output.  Optimization
will never increase either of the removed entries in $\hdrvec{\wsrc}$,
as this would always result in a decrease in the objective.  When $\NEG$
is not empty, we do not remove the 10\% of entries that correspond
to the most frequent words.  These contain the vast majority of the
cooccurrence events, and optimization is most likely to increase them
and not the others.

This algorithm typically runs in minutes, as reported in
Section~\ref{sec:runtimefigures}.

 \iffull{}
 {
 \paragraphbe{The greedy approach is appropriate for objectives with diminishing returns.}
Our objective $\objhdr$ performs a $\log$ operation on entries of $\changevec$ for computing the new (post-attack) $\vec{\mchanged}_{\wsrc}$ entries. We thus expect $\objhdr$ to have diminishing returns, i.e., we expect that as $\changevec$ entries are increased by our optimization procedure, increasing them further will yield lower increases in corresponding $\vec{\mchanged}_{\wsrc}$ entries, and, resultantly, lower increases in $\objhdr$'s value. 

To test this intuition, we performed the following: during each step of the optimization procedure, we recorded the \textit{return values} $\forall \stepind,\optcvecdiff: \stepdiff{\objhdr}$, i.e., the increase in the objective that would occur by setting $\changevec_{\stepind}\gets \changevec_{\stepind}+\optcvecdiff$. We counted the number of values that were positive in the previous step, and the fraction of those values that decreased or did not change in the current step (after updating one of $\changevec$'s entries). We averaged our samples across the runs of the optimization procedure for the 100 word pairs in $\wordpairset$. The number of iteration steps for a word pair ranged from 8,000 to about 20,000, and we measure over the first 10,000 steps.

Figure~\ref{fig:dreturns} shows the results. We observe that the fraction of decreasing return values is typically close to 1, which is congruent with diminishing-returns behavior. As iterations advance, some $\stepdiff{\objhdr}$ entries become very small, and numerical computation errors might explain why the fraction becomes lower.

\begin{figure}[t]
	\includegraphics[width=0.5\textwidth]{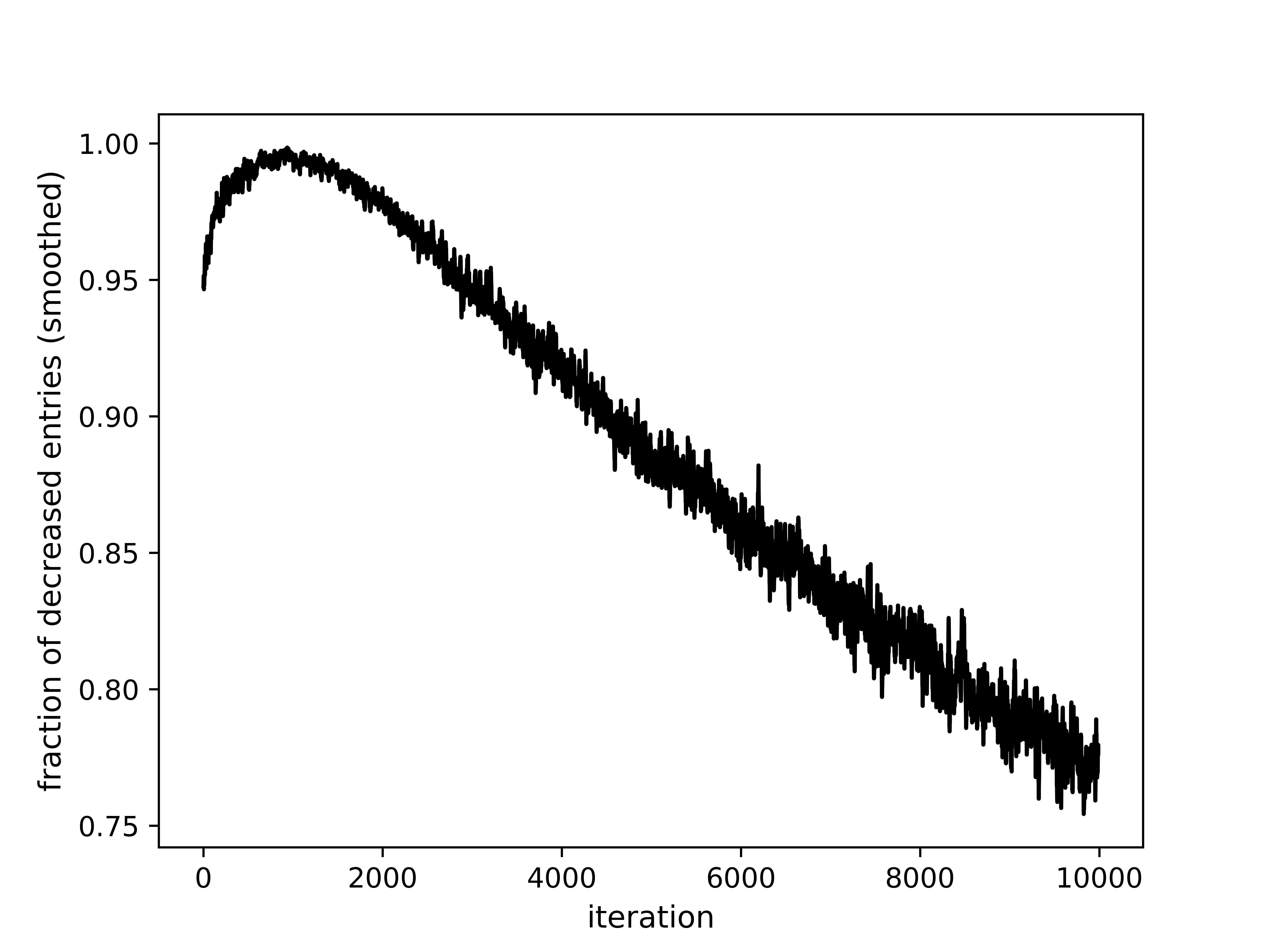}
	\caption{The average fraction of diminishing returns in the first 10,000 iteration steps. The graph was smoothed by averaging over a 10-iteration window. Parameters are the same as Figure~\ref{preservedistance}, but with $\cossimexp=\costwo$.\done\roei{SANITIZED}\label{fig:dreturns}} 
\end{figure}

In Appendix~\ref{sec:approxsubmod} we explain the theoretical guarantee attained when using submodular objectives, which are defined by having diminishing returns. With these objectives, our greedy approach is provably effective. While our objective is not analytically shown to be submodular, we conjecture that the greedy algorithm is particularly appropriate, due to its diminishing returns property.

\fi{}

\paragraphbe{The \textsc{compDiff} sub-procedure.}  the \underline{input}
is $\stepind, \optcvecdiff$, and the saved intermediate computation
states. The state contains (1) the dot product of $\wsrc$'s distributional vector
$\vec{\mexp_{\wsrc}}$ with those of the target words in $\POS\cup\NEG$;
(2) the squared L2 norms of the source and target words' distributional vectors, and
(3) for every word, the sum of its cooccurrence counts with all other words.

We use the following notations: we denote by $\coocchanged\equiv \updatedcoocs{\changevec}$ the cooccurrence matrix after adding $\changevec$ to $\vec{C_{\wsrc}}$. We similarly denote the updated bias terms by $\biaschanged$ and the updated distributional matrix by $\mchanged$. 
We define $\ffirstchanged$ as $f$ computed using the updated bias terms $\curbrac{\biaschanged_{\wone}}_{\wone\in\dict}$, instead of $\curbrac{B_{\wone}}_{\wone\in\dict}$. Finally, let $\explicit{X}$ a distributional expression that depends on $\changevec$, then let $\stepnext{\explicit{X}}\equiv \explicit{X}+\stepdiff{\explicit{X}}$, i.e., the value of the expression after setting $\changevec_\stepind\gets \changevec_\stepind + \optcvecdiff$.

The first step is to compute the updated cooccurrence sums $\stepnext{\ecoocsums{\wone}}$ for $\wone\in\curbrac{\wsrc,\stepind}\cup\POS\cup\NEG$, by adding $\optcvecdiff$ to the appropriate sums (for example, $\stepnext{\ecoocsums{\wsrc}}$ is always updated, since we always add cooccurrences with the source word).
Next, using the updated cooccurrence sums, we compute the updated bias terms $\stepnext{\biaschanged_{\wone}}$ for $\wone\in\curbrac{\wsrc,\stepind}\cup\POS\cup\NEG$.
For $\SPPMI$, these terms also depend
on $\log \brac{Z}$ (see Section~\ref{foot:sppmibias}), which is not a part of our state, but changes in this term are negligible. For $\BIAS$, we use an approximation
explained below.

Using the above, we can compute $\stepnext{\ffirstchanged_{\wsrc,\wtrg}\brac{\ecoocsums{\wone},e^{-60}}}$ for $\wone\in\curbrac{s}\cup\POS\cup\NEG$ and $\stepnext{\ffirstchanged_{\wsrc,\wtrg}\brac{\coocchanged_{\wsrc,\wtrg},0}}$.

Now, we compute the updates to our saved intermediate state. First, we compute $\stepdiff{{\vec{\mchanged_{\wsrc}}}_{\stepind}}$, i.e., the difference in $\vec{\mchanged_{\wsrc}}$'s $\stepind$th entry. This is similar to the previous computation, since matrix entries are computed using $\ffirstchanged$. We use these values, along with {\small$\brac{\dotprod{\vec{\mchanged_{\wsrc}}}{\vec{\mchanged_{\wtrg}}}}$}, which is a part of our saved state, to compute {\small$\stepnext{\brac{\dotprod{\vec{\mchanged_{\wsrc}}}{\vec{\mchanged_{\wtrg}}}}}\gets \brac{\dotprod{\vec{\mchanged_{\wsrc}}}{\vec{\mchanged_{\wtrg}}}} + \stepdiff{{\vec{\mchanged_{\wsrc}}}_{\stepind}}\cdot {\mexp_{\wtrg}}_{\stepind}$} for each target. If $\stepind\in{\POS\cup\NEG}$, we also add a similar term accounting for $\stepdiff{{\vec{\mchanged_{\wtrg}}}_{\wsrc}}$.
We similarly derive $\stepdiff{\vec{\mchanged_{\wsrc}}_{\stepind}^2}$ and use it to compute {\small$\stepnext{\norm{\vec{\mchanged_{\wsrc}}}_2^2}\gets \norm{\vec{\mchanged_{\wsrc}}}_2^2+\stepdiff{{\vec{\mchanged_{\wsrc}}}_{\stepind}^2}$}. If $\stepind\in \POS\cup\NEG$, we similarly compute {\small $\stepnext{\norm{\vec{\mchanged_{\stepind}}}_2^2}$}. For $\SPPMI$, the above does not account for minor changes in bias values of the source or target which might affect all entries of vectors in $\curbrac{\vec{\mchanged}_{\wone}}_{\wone\in\curbrac{\wsrc}\cup\POS\cup\NEG}$. We could avoid carrying the approximation error to the next step (at a minor, non-asymptotical performance hit) by changing Algorithm~\ref{algo:greedy} to recompute the state from the updated cooccurrences at each step, instead of the updates at lines 18-19, but our current implementation does not.

Now we are ready to compute the differences in $\updcosone\brac{\wsrc, \wtrg},\updcostwo\brac{\wsrc,
\wtrg},\updcosboth\brac{\wsrc, \wtrg}$, the distributional
expressions for the first-order, second-order, and combined proximities,
respectively, using $\updatedcoocs{\changevec}$.
For each target:
\begin{tiny}
\begin{flalign*}
	\stepdiff{\updcosone\brac{\wsrc, \wtrg}} \gets &
	\Bigg\{ \frac{\stepnext{\ffirstchanged_{\wsrc,\wtrg}\brac{\coocchanged_{\wsrc,\wtrg},0}}}
	{\sqrt{\stepnext{\ffirstchanged_{\wsrc,\wtrg}\brac{\sum_{\wthr}\coocchanged_{\wsrc,\wthr},e^{-60}}}\stepnext{\ffirstchanged_{\wsrc,\wtrg}\brac{\sum_{\wthr}\coocchanged_{\wtrg,\wthr},e^{-60}}}}}-\\
	& \frac{{\ffirstchanged_{\wsrc,\wtrg}\brac{\coocchanged_{\wsrc,\wtrg},0}}}
	{\sqrt{{\ffirstchanged_{\wsrc,\wtrg}\brac{\sum_{\wthr}\coocchanged_{\wsrc,\wthr},e^{-60}}}{\ffirstchanged_{\wsrc,\wtrg}\brac{\sum_{\wthr}\coocchanged_{\wtrg,\wthr},e^{-60}}}}}\Bigg\}
\end{flalign*}
\end{tiny}


\begin{scriptsize}
\begin{flalign*}
	&\stepdiff{\updcostwo\brac{\wsrc, \wtrg}} \gets 
	\frac{\stepnext{\dotprod{\vec{\mchanged_{\wsrc}}}{\vec{\mchanged_{\wtrg}}}}}
	{\sqrt{\stepnext{\norm{\vec{\mchanged_{\wsrc}}}_2^2}\stepnext{\norm{\vec{\mchanged_{\wtrg}}}_2^2}}} - 
	\frac{\dotprod{\vec{\mchanged_{\wsrc}}}{\vec{\mchanged_{\wtrg}}}}
	{\sqrt{\norm{\vec{\mchanged_{\wsrc}}}_2^2\norm{\vec{\mchanged_{\wtrg}}}_2^2}}&
\end{flalign*}
\end{scriptsize}

and, using the above,

\begin{scriptsize}
$$
\stepdiff{\updcosboth\brac{\wsrc, \wtrg}} \gets \frac{\stepdiff{\updcosone\brac{\wsrc, \wtrg}}+\stepdiff{\updcostwo\brac{\wsrc, \wtrg}}}{2}
$$
\end{scriptsize}

Finally, we compute ${\stepdiff{\objhdr}}$ as
\begin{scriptsize}
\begin{align*}
	\stepdiff{\objhdr}\gets
	 \frac{1}{\size{\POS\cup\NEG}}\cdot &\\
	 \bigg(\sum_{\wtrg\in \POS} \stepdiff{\cossimexp_{\changevec}\brac{\wsrc, \wtrg}}- 
	\sum_{\wtrg\in \NEG} &\stepdiff{\cossimexp_{\changevec}\brac{\wsrc, \wtrg}}\bigg)
\end{align*}
\end{scriptsize}

We \underline{return} ${\stepdiff{\objhdr}}$ and the
computed differences in the saved intermediate values.


\paragraphbe{Estimating biases.}
\label{sec:biasapprox}
When the distributional proximities in $\objhdr$ are computed using
$\mexp=\BIAS$, there is an additional subtlety.  We compute $\BIAS$
using the biases output by \gl{} when trained on the original corpus.
Changes to the cooccurrences might affect biases computed on the modified
corpus.  This effect is likely insignificant for small modifications to
the cooccurrences of the existing words.  New words introduced as part
of the attack do not initially have biases, and, during optimization,
one can estimate their post-attack biases using the average biases of
the words with the same cooccurrence counts in the existing corpus.
In practice, we found that post-retraining $\BIAS$ distributional
distances closely follow our estimated ones (see Figure~\ref{fig:pvspp}).

\newif\ifwolsey
\wolseyfalse
\iffull{}
\subsection{Approximation guarantee for submodular objectives}
\label{sec:approxsubmod}
In this section we show that, under simplifying assumptions, the greedy approach attains an approximation guarantee.

\paragraphbe{Simplifying assumptions.}
Most importantly, we will assume that a proxy function defined using our explicit objective $\objhdr$ is a submodular function (see below for the formal statement).
This is not true in practice, however, the objective is characterized by having diminishing returns, which are the defining property of submodular functions (see Appendix~\ref{sec:analyticsolvedetails}). We also assume for simplicity that $\changesetsize=\norm{\changevec}_1$ (this is true up to a multiplicative constant, except when using $\cossimexp=\cosboth$), that in Algorithm~\ref{algo:greedy} we set $\mathbb{L}\gets \curbrac{1/5}$, and that $\changevec$ is limited to the domain $\mathcal{A}\equiv \curbrac{x\in \mathbb{R}\mid \exists j\in \mathbb{N} : x=(1/5)j}$ (entries are limited to the ones our algorithm can find due to the definition of $\mathbb{L}$).

	 \begin{definition}
		 Let $\mathcal{S}$ a finite set. Then a \textit{submodular set function} is a function $\upsilon:\mathcal{P}\brac{\mathcal{S}}\rightarrow \mathbb{R}$ such that for any $(X,Y)\subseteq \mathcal{S}$ with $X\subseteq Y$, for every $x\in \mathcal{S}\setminus Y$ it holds that:
		 $$
		 \upsilon\brac{X\cup\curbrac{x}}-\upsilon\brac{X}\geq \upsilon\brac{Y\cup\curbrac{x}}+\upsilon\brac{Y}
		 $$
	 \end{definition}
	 Let $\wone\in\dict$ be a word, and $\mathcal{W}_{\wone}=\curbrac{\wone_{1}, ..., \wone_{1000000}}$\footnote{The theorems we rely on require that the set be finite; we choose it as big enough such that in practice it is equivalent to an infinite set.} a set of elements corresponding to $\wone$. We define $\mathcal{S}=\bigcup_{\wone\in \dict} \mathcal{W}_{\wone}$. We define a mapping $\xi$ between subsets of $S$ and change vectors, by $\brac{\xi\brac{X}}_{\wone}\equiv \brac{(1/5)\size{\mathcal{W}_{\wone}\cap \mathcal{S}}}$. Let $\phi\brac{X}\equiv \explicit{\objective}\brac{\wsrc, \NEG, \POS; \xi\brac{X}}$. 


	 \ifwolsey{}
 \begin{theorem}
	 \label{th:thm1}
	 Assume that $\phi$ is nonnegative, monotone increasing in $\changevec$ entries within $\mathcal{A}$, and submodular. Let $SOL_{\changeset}$ be the size of the change attained by the rank attacker with $\costwo$, using the singleton variant. Let $OPT_{\changesetsize}\equiv \min_{\changevec\in \mathcal{A}^n,\objhdr \geq \explicit{\tldr}+\smargin}{\size{\changeset}}$ be the optimal solution over $\mathcal{A}$. Let $\laststepdiff$ be the last objective-difference computed by the attacker's greedy algorithm, i.e., the difference $\chstepdiff{\objhdr}$ value used in line 16 of Algorithm~\ref{algo:greedy}, in its last iteration. Then:
	 $$
	 SOL_{\changesetsize} \leq \brac{1+\log\brac{\frac{1}{\laststepdiff}}}OPT_{\changesetsize}
	 $$
 \end{theorem}
 \fi{}

 \begin{theorem}
	 \label{th:thm2}
	 Assume that $\phi$ is nonnegative, monotone increasing in $\changevec$ entries within $\mathcal{A}$, and submodular. Let $SOL_{\objective}$ be the increase in $\objhdr$ attained by the proximity attacker with $\costwo$, using the singleton variant.
	 Let $OPT_{\objective}\equiv \max_{\changevec\in \mathcal{A},\size{\changeset}\leq \changesetmax}{\objhdr}$ be the value attained in the optimal solution where $\mathcal{A}$ is defined as above. Then:
	 $$
	 SOL_{\objective} \geq \brac{1-1/e}OPT_{\objective}
	 $$
 \end{theorem}

 \begin{proof}
	 We will rely on well-known results for the following greedy algorithm.
	 \begin{definition}
		 The $\textsc{greedySet}(\upsilon, cn)$ algorithm operates on a function $\upsilon:\mathcal{P}\brac{\mathcal{S}}\rightarrow \mathbb{R}$, and a constraint $cn:\mathcal{P}\brac{S}\rightarrow\curbrac{T,F}$. The algorithm is as follows: (1) initiates $X\gets \emptyset$, and (2) iteratively sets $X\gets X\cup\curbrac{\amax_{e\in \mathcal{S}}\brac{\upsilon\brac{X\cup\curbrac{e}}}}$\footnote{Ties in $\amax$ are broken arbitrarily.} until $cn\brac{X}=F$, and (3) then returns $X\setminus \curbrac{e}$ where $e$ is the last chosen element.
	 \end{definition}
	 This algorithm has several guarantees when $\upsilon$ is nonnegative, monotone, and submodular. Particularly, 
	 \ifwolsey{}
	 for \textit{value constraints}, of the form $\upsilon\brac{X}\leq T$ for some constant $T$, it attains~\cite{wolsey1982analysis} a $\brac{1+\log\brac{\frac{1}{d}}}$ multiplicative approximation for the smallest possible set size $OPT_{\size{X}}$ under the constraint, where $d$ is the gain attained in the last optimization step (similar to $d_{last}$ above).
	 \fi{}
	 for \textit{cardinality constraints}, of the form $\size{X}\leq T$, we know~\cite{nemhauser1978best} that the algorithm attains a $\brac{1-1/e}$ multiplicative approximation for the highest possible value of $\upsilon\brac{X}$ under the constraint, which we denote by $OPT_{\upsilon\brac{X}}$. 

	 We analyze the following algorithm, which is equivalent to Algorithm~\ref{algo:greedy}:
	 \ifwolsey{}
	 for the rank attacker, we run $\textsc{greedySet}$ on $\phi$ with a value constraint $\phi\brac{X}\geq\explicit{\tldr}+\smargin$.
	 \fi{}
	 For the proximity attacker, we run $\textsc{greedySet}$ on $\phi$ with a cardinality constraint $(1/5)\size{X}\leq \changesetmax$. We output $\xi\brac{X}$ where $X$ is $\textsc{greedySet}$'s output.

	 \ifwolsey{}
	 \begin{claim}
		 Let $OPT_{\size{X}}$ be the optimal solution for minimizing $\size{X}$ under the value constraint $\phi\brac{X}\geq\explicit{\tldr}+\smargin$. Then $(1/5)OPT_{\size{X}}=OPT_{\changesetsize}$.
	 \end{claim}
	 Let $X$ the optimal solution such that $OPT_{\size{X}}=\size{X}$ and $\phi\brac{X}\geq \explicit{\tldr}+\smargin$. From the latter, have that $\explicit{\objective}\brac{\wsrc, \NEG, \POS; \xi\brac{X}}\geq \explicit{\tldr}+\smargin$. For this solution, $\changesetsize=\norm{\xi\brac{X}}=(1/5)OPT_{\size{X}}$ and thus the optimal value is bounded by $OPT_{\changesetsize}\leq (1/5)OPT_{\size{X}}$.

	 Let $\changevec$ the optimal solution such that $OPT_{\changesetsize}=\changesetsize=\norm{\changevec}_1$ and $\explicit{\objective}\brac{\wsrc, \NEG, \POS; \changevec}\leq \explicit{\tldr}+\smargin$.
	 We define an inverse mapping $\xi^{-1}\brac{\changevec}={\bigcup}_{\wone\in\dict} \curbrac{u_{1}, ..., u_{5\cdot \changevec_{\wone}}}$ (for $\changevec\in \mathcal{A}$, entries multiplied by 5 are integers). Note that $\xi\brac{\xi^{-1}\brac{\changevec}}=\changevec$.
	 We know that $\explicit{\objective}\brac{\wsrc, \NEG, \POS; \changevec}\leq \explicit{\tldr}+\smargin$. Thus $\phi\brac{\xi^{-1}\brac{\changevec}}\leq \explicit{\tldr}+\smargin$. For $X=\xi^{-1}\brac{X}$, we know that $\size{X}/5=OPT_{\changevec}$. Thus, the optimal value is bounded by $OPT_{\changesetsize}\geq (1/5)OPT_{\size{X}}$.

	 Theorem~\ref{th:thm1} directly follows from the claim and the above-mentioned approximation guarantee due to Wosley~\cite{wolsey1982analysis}.
	 \fi{}
	 \begin{claim}
		 Let $OPT_{\phi\brac{X}}$ be the optimal solution for maximizing $\phi\brac{X}$ with a cardinality constraint $(1/5)\size{X}\leq \changesetmax$. Then $OPT_{\phi\brac{X}}=OPT_{\objective}$.
	 \end{claim}

	 Let $X$ be the solution such that $\phi\brac{X}=OPT_{\phi\brac{X}}$ and $(1/5)\size{X}\leq \changesetmax$. Since $\phi\brac{X}=\explicit{\objective}\brac{\wsrc, \NEG, \POS; \xi\brac{X}}$ and $\size{\changeset}=(1/5)\size{X}\leq \changesetmax$, we have that $OPT_{\phi\brac{X}}\leq OPT_{\objective}$.

	 Let $\changevec$ the solution such that $\explicit{\objective}\brac{\wsrc, \NEG, \POS; \xi\brac{X}}=OPT_{\objective}$. Again, we use the fact that $\xi\brac{\xi^{-1}\brac{\changevec}}=\changevec$ and get that $\explicit{\objective}\brac{\wsrc, \NEG, \POS; \xi\brac{\xi^{-1}\brac{\changevec}}}=OPT_{\objective}$. Moreover, $\size{\xi^{-1}\brac{\changevec}}=5\cdot \changesetsize$, so $(1/5)\size{\xi^{-1}\brac{\changevec}}\leq\changesetmax$. Thus, $OPT_{\phi\brac{X}}\geq OPT_{\objective}$.

	 Theorem~\ref{th:thm2} directly follows from the claim and the above-mentioned approximation guarantee due to Nemhauser and Wolsey~\cite{nemhauser1978best}.
 \end{proof}

 }

 \fi{}

\subsection{Placement strategy (details)}
\label{sec:aplacement}

As discussed in Section~\ref{sec:placement}, our attack involves adding
two types of sequences to the corpus.

\paragraphbe{First-order sequences.}
For each $\wtrg\in \POS$, to increase $\cosone$ by the required
amount, we add sequences with exactly one instance of $\wsrc$
and $\wtrg$ each until the number of sequences is equal to
$\left\lceil\changevec_{\wtrg}/\coocweight\brac{1}\right\rceil$, where
$\coocweight$ is the cooccurrence-weight function.

We could leverage the fact that $\coocweight$ can count multiple
cooccurrences for each instance of $\wsrc$, but this has disadvantages.
Adding more occurrences of the target word around $\wsrc$ is
pointless because they would exceed those of $\wsrc$ and dominate
$\changesetsize$, particularly for pure $\cosone$ attackers with just one
target word.\footnote{This strategy might be good when using $\cosboth$
or when $\size{\POS\cup\NEG}>1$, because occurrences of $\wsrc$ exceed
those of $\wtrg$ to begin with, but only under the assumption that adding
many cooccurrences of target word with itself does not impede the attack.
In this paper, we do not explore further if the attack can be improved in
these specific cases.} We thus require symmetry between the occurrences
of the target and source words.

Sequences of the form $\wsrc\ \wtrg\ \wsrc\ \wtrg\ldots$ could increase
the desired extra cooccurrences per added source (or target) word by a
factor of 2-3 in our setting (depending on how long the sequences are).
Nevertheless, they are clearly anomalous and would result in a fragile
attack.  For example, in our Twitter corpus, sub-sequences of the form
$X\ Y\ X\ Y$ where $X\neq Y$ and $X, Y$ are alpha-numeric words, occur
in 0.03\% of all tweets.  Filtering out such rare sub-sequences would
eliminate 100\% of the attacker's first-order sequences.

We could also merge $t$'s first-order appearances with those of
other targets, or inject $t$ into second-order sequences next to $s$.
This would add many cooccurrences of $\wtrg$ with words other than
$\wsrc$ and might \emph{decrease} both $\cosone\brac{\wsrc, \wtrg}$
and $\costwo\brac{\wsrc, \wtrg}$.


\paragraphbe{Second-order sequences.}
We add 11-word sequences that include the source word $\wsrc$ and 5
additional on each side of $\wsrc$.  Our placement strategy forms these
sequences so that the cooccurrences of $\wone\in \dict\setminus\POS$
with $\wsrc$ are approximately equal to those in the change vector
$\changevec_{\wone}$.  This has a collateral effect of adding
cooccurrences of $\wone$ with words other than $\wsrc$, but it does not
affect $\cosone\brac{\wsrc, \wtrg}$, nor $\costwo\brac{\wsrc, \wtrg}$.
Moreover, it is highly unlikely to affect the distributional proximities
of the added words $\wone\not\in\POS\cup\curbrac{s}$ with \emph{other}
words since, in practice, every such word is added at most a few times.

We verified this using one of our benchmark experiments
from Section~\ref{sec:benchmarks}.  For solutions found with
$\cossimexp=\simtwo,\mexp=\BIAS,\changesetsize=1250$, only about 0.3\%
of such entries $\changevec_{\wone}$ were bigger than 20, and, for 99\%
of them, the change in $\norm{\vec{C_{\wone}}}_1$ was less than 1\%.
We conclude that changes to $\vec{C_{\wone}}$ where $\wone$ is neither
the source nor the target have negligible effects on distributional
proximities.

\paragraphbe{Placement algorithm.}
Algorithm~\ref{algo:placement} gives the pseudo-code of our
placement algorithm.  It constructs a list of sequences, each with
$2\cdot\windowsize+1$ words for some $\windowsize$ (in our setting,
5), with $\wsrc$ in the middle, $\windowsize$th slot.  Since the
sum of $\wsrc$'s cooccurrences in each such sequence is $2\sum_{d\in
\sqbrac{\windowsize}}\coocweight\brac{d}$, we require a minimum of
$\brac{\sum_{\wone\in\dict\setminus\POS}\curbrac{\changevec_{\wone}}}/\sum_{d\in
\sqbrac{\windowsize}}\coocweight\brac{d}$ sequences.

After instantiating this many sequences, the algorithm traverses
$\changevec$'s entries and, for each, inserts the corresponding word into
the non-yet-full sequences until the required number of cooccurrences
is reached.  For every such insertion, it tries to find the slot whose
contribution to the cooccurrence count most tightly fits the needed value.
After all sequences are filled up, a new one is added.  In the end,
sequences that still have empty slots are filled by randomly chosen
words that have nonzero entries in $\changevec$ (but are not in $POS$).
We found that this further improves the distributional objective without
increasing $\changesetsize$.  Finally, $\changeset$ is added to the
corpus.

\paragraphbe{Properties required in Section~\ref{sec:placement} hold.} 
First, assume that $\changevec$ has entries corresponding to either
first- or second-order sequences but not both.  Observe that in our
$\changesetsize$, $\wsrc$ is always the word with the most occurrences,
and it occurs in each sequence once.  Therefore, $\changesetsize$ is
always equal to the number of sequences and the number of source-word
occurrences in $\changeset$ (see the definition of $\changesetsize$
in Section~\ref{sec:methodology}).

For $\changevec$ with only first-order changes,
both properties trivially hold, because we add
$\left\lceil\changevec_{\wtrg}/\coocweight\brac{1}\right\rceil$
cooccurrences of the source and the target.  The size of the change
is thus predictable, as it adds almost exactly $\changevec_{\wsrc}$
to $\changesetsize$.

For second-order changes, both properties empirically hold.
First, $\changesetsize$ is still linear in $\size{\changevec}$:
their Pearson correlation is over 0.99 for the rank attacker
in Section~\ref{sec:searchattack}, where $\changesetsize$ varies.
Thus, $\changesetsize$ is a constant multiple of $\size{\changevec}$
and close to optimal.  For example, it is about 4 times smaller than
$\norm{\changevec}_1$ for the \gl{} attack, the optimal value being
$2\sum_{d\in \sqbrac{\windowsize}}\coocweight\brac{d}\approx 4.5$.
Second, for the proximity attacker in Section~\ref{sec:benchmarks},
where $\changesetsize$ is constant but $\objhdr$ varies, we measured
>0.99 Pearson correlation between the proximities attained by $\objhdr$
and those computed over the actual, post-placement cooccurrence counts
(see Figure~\ref{fig:pvsp}).


If $\changevec$ contains both first- and second-order entries (because
the objective uses $\cosboth$), the aggregate contribution to the
cooccurrence counts still preserves the objective's value because it
separately preserves its $\cosone$ and $\costwo$ components.  We can
still easily compute $\changesetsize$ via their weighted sum (e.g.,
divide second-order entries by 4 and first-order entries by 1).

\begin{center}
	\begin{spacing}{0.8}
	\begin{algorithm}
		\captionof{algorithm}{Placement into corpus: finding the change set $\changeset$\label{algo:placement}}
		\scriptsize
		\begin{algorithmic}[1]
\Procedure{PlaceAdditions}{vector $\changevec$, word $\wsrc$}
\State $\changeset \gets \emptyset$
			\ForEach {$\wtrg\in \POS$} // First, add first-order sequences
				\State $\changeset \gets \changeset \cup \underbrace{\curbrac{``\wsrc\ \wtrg'', ..., ``\wsrc\ \wtrg''}}_{\times\left\lceil\changevec_{\wtrg}/\coocweight\brac{1}\right\rceil}$
			\EndFor

		\State // Now deal with second-order sequences
			\State $\textrm{\changeMap}\gets \curbrac{\wone\rightarrow\changevec_{\wone}\mid \changevec_{\wone}\neq 0 \land \wone\not\in \POS}$
			\State $\textrm{minSequencesRequired}\gets \left\lceil\frac{\sum_{\wone\in\dict\setminus\POS}\changevec_{\wone}}{\sum_{d\in \sqbrac{\windowsize}}\coocweight\brac{d})}\right\rceil$

			\State $\live\gets \underbrace{\curbrac{\newsequence, ..., \newsequence}}_{\times\textrm{minSequencesRequired}}$ 

		\State $\indices\gets \curbrac{-\injwindsize, -4, -3, -2, -1, 1, 2, 3, 4, \injwindsize}$ \\
	
		\ForEach{$\wone \in \changeMap$}
			\While {$\changeMap\sqbrac{\wone} > 0$}
				\State $\seq, i \gets$
				$ \am_{\substack{\seq\in \live,\\ i \in indices\\s.t. \seq\sqbrac{i} = "\emptyspace"}} \Big\lvert\coocweight\brac{\size{i}} - \changeMap\sqbrac{\wone}\Big\rvert$
				\State $\seq\sqbrac{i} \gets \wone$
				\State $\changeMap\sqbrac{\wone} \gets \changeMap\sqbrac{\wone} - \coocweight\brac{\size{i}}$
				\If {$\forall i \in \indices: \seq\sqbrac{i}\neq"\emptyspace"$}
					\State $\changeset \gets \changeset \cup \curbrac{\seq}$
					\State $\live \gets \live \setminus \curbrac{\seq}$
				\EndIf
	 			\If {$\live = \emptyset$}
					\State $\live\gets \curbrac{\newsequence}$
				\EndIf
			 \EndWhile
		\EndFor
\State // Fill empty sequences with nonzero $\changevec$ entries
\ForEach{$\seq\in \live$}
\ForEach{$i\in \curbrac{i \in \indices\mid \seq\sqbrac{i}="\emptyspace"}$}
			\State $\seq\curbrac{i}\gets \textrm{RandomChoose}\brac{\curbrac{\wone\in\changeMap}}$
\State 
\EndFor
			\State $\changeset \gets \changeset\cup\curbrac{\seq}$
\EndFor
\Return $\changeset$
\EndProcedure
\end{algorithmic}
	\end{algorithm}
\end{spacing}
\end{center}


{
\begin{figure*}[t]
\makeatletter
\renewcommand{\p@subfigure}{\thefigure--}
\makeatother
\hfil
\captionsetup[subfigure]{width=0.31\textwidth}
	\subfloat[
		Post-placement distributional proximities
		]{
	\includegraphics[width=0.31\textwidth]{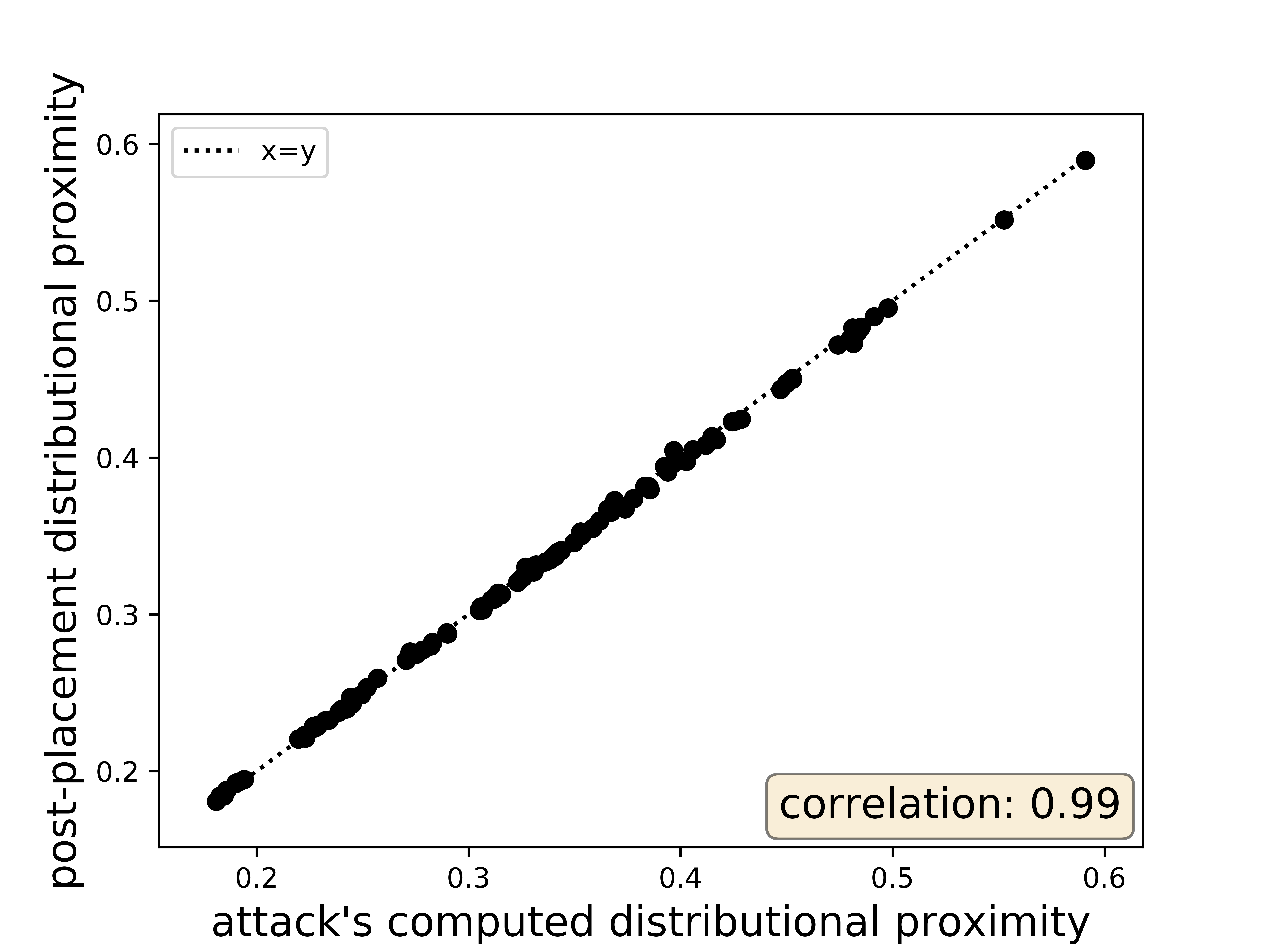}
	\label{fig:pvsp}
	}
\hfil
	\subfloat[
		Post-retraining distributional proximities
		]{
	\includegraphics[width=0.31\textwidth]{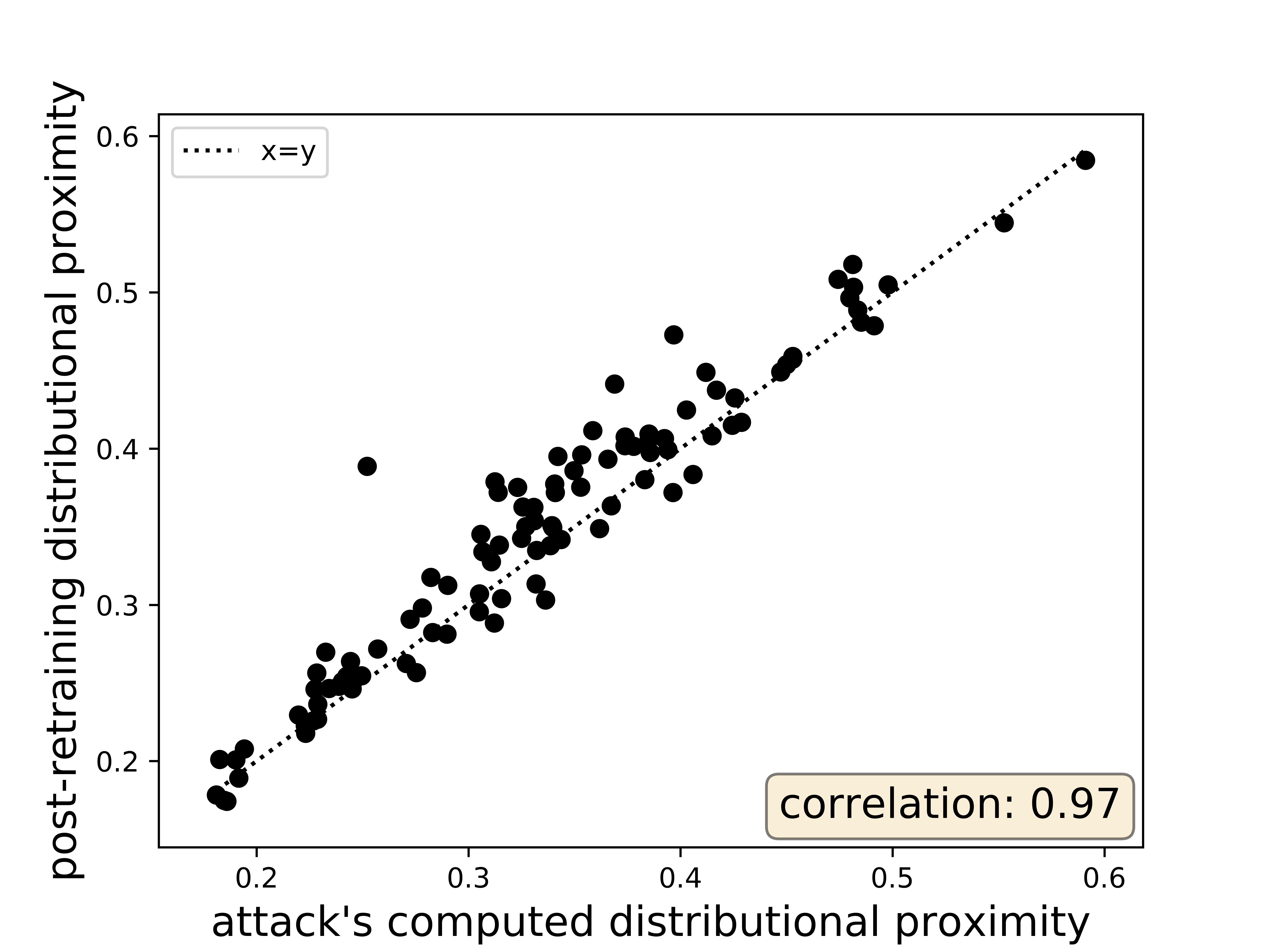}
	\label{fig:pvspp}
	}
\hfil
\subfloat[
		Final embedding proximities
	]{
	\includegraphics[width=0.31\textwidth]{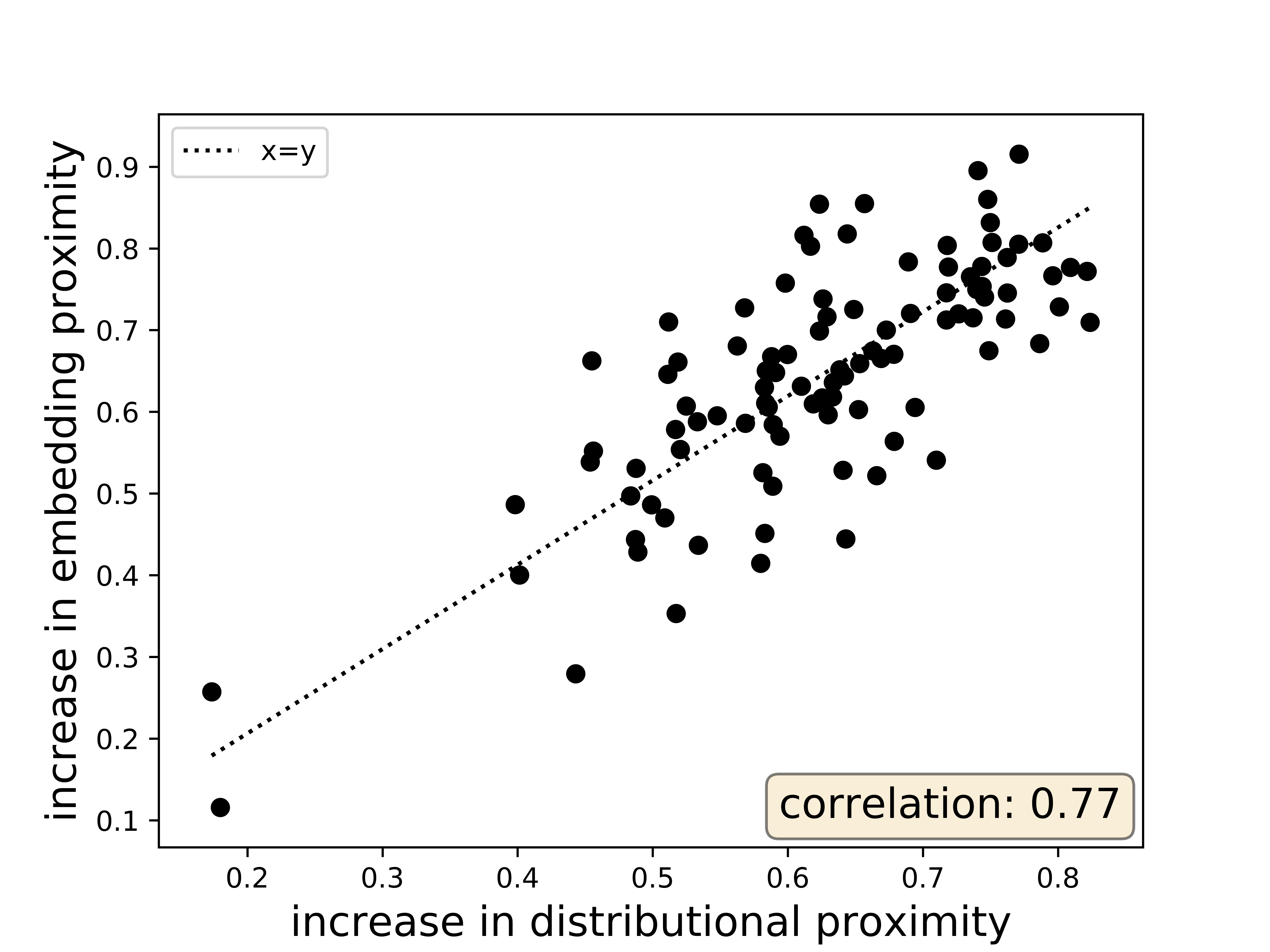}
	\label{fig:pvse}
	}
\caption{Comparing the proxy distances used by the attacker with the post-attack distances in the corpus for the words in $\wordpairset$, using \gl-\texttt{paper}/Wikipedia, $M=\BIAS,\cossimexp=\cosboth,\changesetmax=1250$.\label{preservedistance}} 
\end{figure*}
}

\label{sec:validating}

\subsection{Distributional distances are good proxies}
\label{sec:evaluatingplacement}

Figure~\ref{preservedistance} shows how distributional
distances computed during the attack are preserved throughout
placement (Figure~\ref{fig:pvsp}), re-training by the victim
(Figure~\ref{fig:pvspp}), and, finally, in the new embedding
(Figure~\ref{fig:pvse}).  The latter depicts how increases in
distributional proximity correspond roughly linearly to increases in
embedding proximity, indicating that the former is an effective proxy
for the latter.


\subsection{Alternative attack on resume search}
\label{sec:altsearchattack}

In this section, we consider an attacker who\textemdash instead of
poisoning the embedding\textemdash changes his resume to include words
from the expanded queries.  Specifically, he adds the closest neighbor
of $\wtrg$ in the original embedding to his resume so that it is returned
in response to queries where $\wtrg\in \searchset$.

First, this attacker must add a specific \emph{meaningful} word to
his resume.  For example, to match the $K=1$ expanded query ``iOS,''
the attacker needs to add ``Android'' to his resume, but this involves
claiming that he actually has expertise in Android.  By contrast, our
embedding attack adds a made-up string, e.g., a social media handle
or nickname.

Further, this attack significantly underperforms our embedding attack.
We carried out an experiment similar to Section~\ref{sec:searchattack},
but for each target word $\wtrg$, we randomly chose a single resume
and added to it $\wtrg$'s nearest neighbor in the original embedding.
We did not poison 20 resumes, as this would have turned $\wtrg$'s nearest
neighbor into a common word, decreasing its score.  Query expansion
uses the original embedding.  For comparison, we repeated our experiment
from Section~\ref{sec:searchattack}, but modifying just one resume per
$\wtrg\in \searchset$.

The embedding attacker outperforms the alternative attacker for every
$K\in \curbrac{1,2,3}$ and every query type.  Averaged over query types
and $K$ values, the average and median ranks attained by the embedding
attacker are 2.5 times lower (i.e., better) than the alternative
attacker's.

\subsection{Attack with deletions}
\label{sec:deletions}

We now consider an attacker who can delete cooccurrence events from
the corpus.  While this is a stronger threat model, we find that it does
not dramatically improve the trade-off between the size of the changes
to the corpus and the corresponding changes in distributional proximity.

Supporting deletions requires some changes.

\paragraphbe{Attacker.}
First, corpus changes now include events that
correspond to a decrease in cooccurrence counts.  We define
$\changeset=\changeset_{add}\cup\changeset_{rm}$ where $\changeset_{add}$
are the sentences added by the attacker (as before), and $\changeset_{rm}$
are the cooccurrence events deleted by the attacker.

The modified corpus $\corpus+\changeset$ is now defined as $\corpus$
augmented with $\changeset_{add}$ and with the word appearances
in $\changeset_{rm}$ flipped to randomly chosen words.  A word flip
does not delete a cooccurrence event per se but replaces it by another
cooccurrence event between $\wsrc$ and some randomly chosen word $\wone$.
These are almost equivalent in terms of their effect on the distributional
proximities because cooccurrence vectors are very sparse.  In our
Wikipedia corpus, for a random subsample of 50,000 most common words,
we found that on average 1\% of the entries were non-zero.  It is thus
highly likely that $\coocs_{\wsrc,\wone}$ is initially 0 or very low.
If so, then $M_{\wsrc, \wone}$ is likely 0 and will likely remain 0
(due to the $\max$ operation in all of our candidate $\mexp$\textemdash
see Section~\ref{sec:ourApproachHDR}) even after we add this cooccurrence
event.  Therefore, the effect of a word flip on distributional proximities
is similar to word removal.

Let $d_e$ be the distance of the removed word from $\wsrc$
for $e\in \changeset_{rm}$.  Let $\size{\changeset_{rm}}\equiv
\sum_{e \in\changeset_{rm}}{\coocweight\brac{d_{e}}}$ be the sum of
cooccurrence-event weights of $\changeset_{rm}$.  We similarly define
$\size{\changeset_{add}}$ as the weighted sum of cooccurrence events
added to the corpus by $\changeset_{add}$.  Under the $\costwo$ attacker,
where $\vec{\omega}$ entries are identical, and using our placement
strategy, the definition of $\size{\changeset_{add}}$ is equivalent to
the definition of $\changesetsize$, up to multiplication by the value
of $\vec{\omega}$ entries.

We redefine $\changesetsize$ as
$\size{\changeset_{add}}+\beta\size{\changeset_{rm}}$.  Under this
definition, word-flip deletions that are close to $\wsrc$ cost more to the
attacker in terms of increasing $\changesetsize$. $\beta$ is this cost.



\paragraphbe{Optimization in cooccurrence-vector space.}
We modify the optimization procedure from Section~\ref{sec:analyticsolve}
as follows.  First, we set $\mathbb{L}\gets [-5,4]\cap \curbrac{i/5\mid
i\in \mathbb{Z}_{\neq 0}}$. This allows the optimization to add negative
values to the entries in the cooccurrence change vector and to output
$\changevec$ with negative entries.  Second, we apply a different weight
to the negative values by multiplying the computed ``step cost'' value
by $\beta$ (line 14 of Algorithm~\ref{algo:greedy}).

\paragraphbe{Placement strategy.}
We modify the placement strategy from Section~\ref{sec:placement} as
follows.  First, we set $\changevec^{pos}\gets\max\curbrac{\changevec,0}$
(for the element-wise $\max$ operation) and use $\changevec^{pos}$ as
input to the original placement Algorithm~\ref{algo:placement}.  Then, we
set $\changevec^{neg}\gets \min\curbrac{x,0}$ for an element-wise $\min$
operation.  We traverse the corpus to find cooccurrence events between
$\wsrc$ and another word $\wone$ such that $\changevec^{neg}_{\wone}$
is non-zero.  Whenever we find such an event, $\wone$'s location in
the corpus is saved into $\changeset_{rm}$.  We then subtract from
$\changevec^{neg}_{\wone}$ the weight of this cooccurrence event.

\paragraphbe{Evaluation.} 
We use three source-target pairs\textemdash war-peace, freedom-slavery,
ignorance-strength\textemdash with $\beta=1$.  We attack
\gl{}-\texttt{tutorial} trained on Wikipedia with window size of 5 using
a distance attacker, $w_s$ set to the source word in each pair, $POS$
to $w_s$ only, and $\changesetmax=1000$.  We also perform an identically
parameterized attack without deletions.



Table~\ref{tab:results1984} shows the results.  They are almost identical,
with a slight advantage to the attacker who can use deletions.

\begin{table}
             \centering
	     \footnotesize

                                \setlength{\tabcolsep}{3.5pt}
                                \begin{tabular}{ l|r|r}
					$\wsrc-\wtrg$ & \makecell[l]{additions only} & \makecell[l]{additions \& deletions} \\
                                \hline
                                
					war-peace & 0.219480 & 0.219480 \\
					freedom-slavery & 0.253640 & 0.253636 \\
					ignorance-strength & 0.266967 & 0.264050 \\
                                \end{tabular}
				\caption{Attained $\cosboth\brac{\wsrc,\wtrg}$ proximity with and without deletions.} \label{tab:results1984}

\end{table}

\balance
\end{document}